%% file: DualQuant_arxiv.tex
\documentclass{article} 
\usepackage{iclr2027_conference,times}

\input{math_commands.tex}

\usepackage{hyperref}
\usepackage{url}
\usepackage[shortlabels]{enumitem}

\usepackage[utf8]{inputenc} 
\usepackage[T1]{fontenc}    
\usepackage{hyperref}       
\usepackage{url}            
\usepackage{booktabs}       
\usepackage{amsfonts}       
\usepackage{nicefrac}       
\usepackage{microtype}      
\usepackage{xcolor}         
\usepackage{amsmath}
\usepackage{algorithm}
\usepackage{algpseudocode}
\usepackage{multirow}
\usepackage{booktabs}
\usepackage{graphicx}
\usepackage{rotating}
\usepackage{comment}
\usepackage{amsthm}
\newtheorem{theorem}{Theorem}
\newtheorem{lemma}{Lemma}
\newtheorem{assumption}{Assumption}
\newtheorem{remark}{Remark}
\usepackage{float}
\usepackage{placeins}
\usepackage[table]{xcolor}
\usepackage{amssymb}
\usepackage{enumitem}
\usepackage{pifont}
\usepackage{pifont}
\usepackage{rotating}
\usepackage{subcaption}
\usepackage{xurl}           

\definecolor{tickgreen}{HTML}{59A14F}
\newcommand{\cmark}{\textcolor{tickgreen}{\ding{51}}}
\definecolor{xred}{HTML}{e15759}
\newcommand{\xmark}{\textcolor{xred}{\ding{55}}}
\newcommand{\hl}[1]{\cellcolor[HTML]{e3f4f7}#1}

\newcommand{\method}{\textsc{DualQuant}}

\newcommand{\mahsa}[1]{\textcolor{black}{#1}}
\definecolor{editgreen}{HTML}{1a7f37}
\newcommand{\green}[1]{#1}

\title{DualQuant: Dual-scale Calibration-free
Weight-and-Activation Quantization for LLMs}

\author{Mahsa Salmani, Sayeh Sharify \& Hesham Mostafa  \\
d-Matrix\\
}

\iclrfinalcopy 
\begin{document}

\maketitle
\lhead{}
\chead{}
\rhead{}
\renewcommand{\headrulewidth}{0pt}

\begin{abstract}
Low-bit inference on accelerators with native 4-bit support
requires quantizing both weights \emph{and} activations to
realize the full throughput gains these formats provide.
While post-training weight quantization has been extensively studied, 
activation quantization remains considerably more challenging 
and typically {requires} calibration data.
In this work, we introduce \method{}, a dual-scale, calibration-free
post-training quantization framework that directly
minimizes the weight-reconstruction Mean Squared Error (MSE) and,
as an immediate by-product, yields per-channel activation
rescalings prior to activation quantization.
Our key {observation} is that the scaling factors obtained by minimizing weight-reconstruction error 
also provide effective activation rescalings, 
enabling \mahsa{{competitive}} {activation} quantization without calibration data.
We formulate \method{}
as a non-differentiable weight-reconstruction problem and develop an efficient solver based on variable splitting and block coordinate descent, where each subproblem admits a closed-form update.
By treating the quantizer as a black box and
projecting scales onto their \emph{format-specific admissible set} after
each update, \method{} provides a unified algorithm supporting a wide range of formats, including 
INT4, INT8, the full MX family, and NVFP4. To our knowledge, \method{}
is the first dual-scaling method to formally support MX and NV
block formats.
Across Llama and Qwen models evaluated on WikiText-2, C4, and zero-shot downstream tasks, \method{} matches or exceeds the accuracy of calibration-based SmoothQuant across most
models and formats, and substantially outperforms calibration-free baselines, with particularly large gains for MX and NV block formats.
We further validate these gains with a real packed W4A4 inference path for Llama-3.1-8B on an RTX 4090. The resulting implementation closely matches the accuracy of fake quantization while delivering up to $2.61\times$ memory reduction and up to $2.46\times$ throughput speedup over BF16.
Our anonymized codebase is available \href{https://github.com/MahsaMoonlike/dualquant}{here}.

\end{abstract}

\section{Introduction}
\label{sec:intro}
Large Language Models (LLMs) have grown rapidly in scale, from dense
models~\citep{dubey2024llama3, yang2025qwen3} to
mixture-of-experts architectures exceeding a trillion
parameters~\citep{deepseekai2025deepseekv3, jiang2024mixtral},
placing enormous pressure on memory and compute at inference
time. Post-Training Quantization (PTQ) addresses this by
compressing weights, activations, and KV-cache to low-bit
representations without retraining. Modern accelerators such
as NVIDIA Blackwell GPUs provide native 4-bit matrix
multiplication~\citep{rouhani2023microscaling, nvidia2025nvfp4}
delivering up to $4\times$ the throughput of FP16, but
realizing this gain requires \emph{both} weights and
activations in a low-bit format. The central obstacle is that
activation tensors exhibit systematic large-magnitude outliers
in specific channels~\citep{dettmers2022llmint8,
bondarenko2021understanding, sun2024massive} that dominate the
quantization grid. 
Calibration-based methods address activation quantization effectively but require representative data, introducing distribution sensitivity and data-access overhead prohibitive in privacy-sensitive or on-device settings. This includes both scale-based approaches such as SmoothQuant~\citep{xiao2023smoothquant} and rotation-based approaches such as SpinQuant~\citep{liu2025spinquant}, which optimizes rotation matrices using calibration data. Calibration-free rotation-based methods such as QuaRot~\citep{ashkboos2024quarot} avoid this dependency but do not target activation-quantization error. In other words, no prior method achieves competitive calibration-free joint Weight and Activation (W+A) quantization across modern MX and NV formats.

\mahsa{In this work, we introduce \method{}, a calibration-free
dual-scale (per-row and per-column) PTQ framework that jointly
quantizes weights and activations \emph{without any access to
calibration data or activation statistics}.
The key insight is that channels with large-magnitude weight
columns contribute disproportionately to the layer output and,
as shown empirically, concentrate activation
outliers~\citep{sun2024massive,dettmers2022llmint8,bondarenko2021understanding}.
Consequently, the column scales that minimize weight-reconstruction
MSE attenuate these same channels, so they not only quantize the
weights well but also implicitly define activation rescalings
that we observe to substantially improve activation
quantization---playing a role similar to SmoothQuant's
calibrated scales~\citep{xiao2023smoothquant}, yet derived from
weights alone.
\method{} obtains these row and column scales via Block Coordinate
Descent (BCD) with closed-form updates and provable convergence
guarantees~\citep{tseng2001convergence, razaviyayn2013unified},
treating the quantizer as a black-box operator and projecting
scales to their admissible format set after each update, 
yielding a single unified algorithm that operates across INT4, MXINT4, MXFP4, and NVFP4. Evaluation on Qwen and Llama families confirms that \method{}
matches or surpasses SmoothQuant on W+A quantization across most
formats and models, without any calibration data.} Our main contributions are:
\begin{enumerate}[nosep, wide=0pt, leftmargin=*]
\item We formalize dual-scale weight quantization as a
constrained non-convex optimization problem and solve it via
BCD with closed-form per-coordinate updates, format-aware
scale projection, and provable convergence guarantees,
requiring no inner loop, no Hessian, and no calibration
data. The same algorithm operates uniformly across INT4,
MXINT4, MXFP4, and NVFP4, making \method{} the first
dual-scaling method that formally supports the full MX and NV
formats.

\item \mahsa{We show that the column scales obtained from weight MSE minimization also serve as effective activation rescalings, enabling joint W+A quantization without any access to calibration data. Across Qwen and Llama families and four 4-bit formats, \method{} closely matches or surpasses
calibration-based SmoothQuant and significantly outperforms calibration-free baselines QuaRot and SINQ+, with especially large margins on MX and NV block data formats.}

\item We validate \method{} through a real packed W4A4 deployment of Llama-3.1-8B on an RTX 4090. The packed implementation matches fake-quantization accuracy to within $0.01$ perplexity on WikiText-2 while reducing memory usage by $2.51$--$2.61\times$ and improving end-to-end speed by $1.39$--$2.46\times$ over BF16. Incorporating the scales into the activation-quantization kernel avoids additional kernel launches and limits the quantization overhead to $4.3$-$6.1\%$ of total GPU time.

\end{enumerate}

\section{Background and related work}
\label{sec:related}

\subsection{Post-training weight quantization}
A rich body of PTQ methods has been proposed to close the gap
between quantized and full-precision quality:
GPTQ~\citep{frantar2022gptq} compensates for quantization error
via second-order weight correction;
AWQ~\citep{lin2024awq} protects salient channels via
activation-weighted scaling. Calibration-free approaches HQQ~\citep{badri2023hqq} and
QuantEase~\citep{behdin2023quantease} minimize per-layer
reconstruction objectives without requiring calibration data while AdpQ~\citep{adpq2024} minimizes weight
distribution divergence via adaptive LASSO. 
SINQ~\citep{muller2025sinq} applies dual row and column scaling
via Sinkhorn-Knopp iteration~\citep{sinkhorn1967},
minimizing matrix imbalance as a proxy for weight quantization
error.
Despite their variety, these methods quantize weights
only, leaving activations in BF16 and forgoing the full
throughput of low-precision accelerators.

\subsection{Activation quantization}
Activations are significantly harder to quantize than
weights~\citep{dettmers2022llmint8, bondarenko2021understanding}:
they exhibit systematic large-magnitude outliers in specific
channels~\citep{xiao2023smoothquant, wei2022outlier} that
dominate the quantization grid and leave normal channels with
very few effective bits.
LLM.int8()~\citep{dettmers2022llmint8} handles outlier channels
via mixed-precision decomposition.
SmoothQuant~\citep{xiao2023smoothquant} migrates the quantization burden from activations to weights via calibration-based per-channel
rescaling. ZeroQuant~\citep{yao2022zeroquant} applies dynamic per-token activation quantization alongside group-wise weight quantization while OmniQuant~\citep{shao2024omniquant} and
AffineQuant~\citep{ma2024affinequant} jointly optimize scaling
and clipping via learnable transformations.
\mbox{Rotation-based}
methods~\citep{tseng2024quip, ashkboos2024quarot,
sun2025flatquant, czako2025smoothrot, liu2025spinquant, saxena2025resq, turboquant2026} suppress outliers by
spreading activation energy incoherently across dimensions via
Hadamard, orthogonal rotations, or learned affine transforms. 
While effective, most of these approaches depend on calibration
data, motivating the search for calibration-free alternatives
that can achieve competitive W+A quantization.
We propose \method{}, a calibration-free joint W+A quantization method that leverages a principled
MSE optimization of weight quantization via block
coordinate descent with closed-form updates \mahsa{and provable convergence guarantees (Appendix~\ref{app:proof_theory1}),}
extending naturally to MX and NV
format families. Most significantly, we show that the column scales derived from minimizing weight-reconstruction error also serve as effective activation rescalings, enabling joint activation and
weight quantization that matches and surpasses SmoothQuant across
most formats and models, without any access to activations or
calibration data.
As illustrated in
Figure~\ref{fig:motivation}, directly targeting the MSE yields
substantially lower reconstruction error than methods such as
SINQ~\citep{muller2025sinq} that optimize a proxy objective,
and we show in Section~\ref{sec:experiments} that this
improvement translates directly to better downstream
performance.

\begin{figure}
    \centering
    \includegraphics[width=0.65\linewidth]{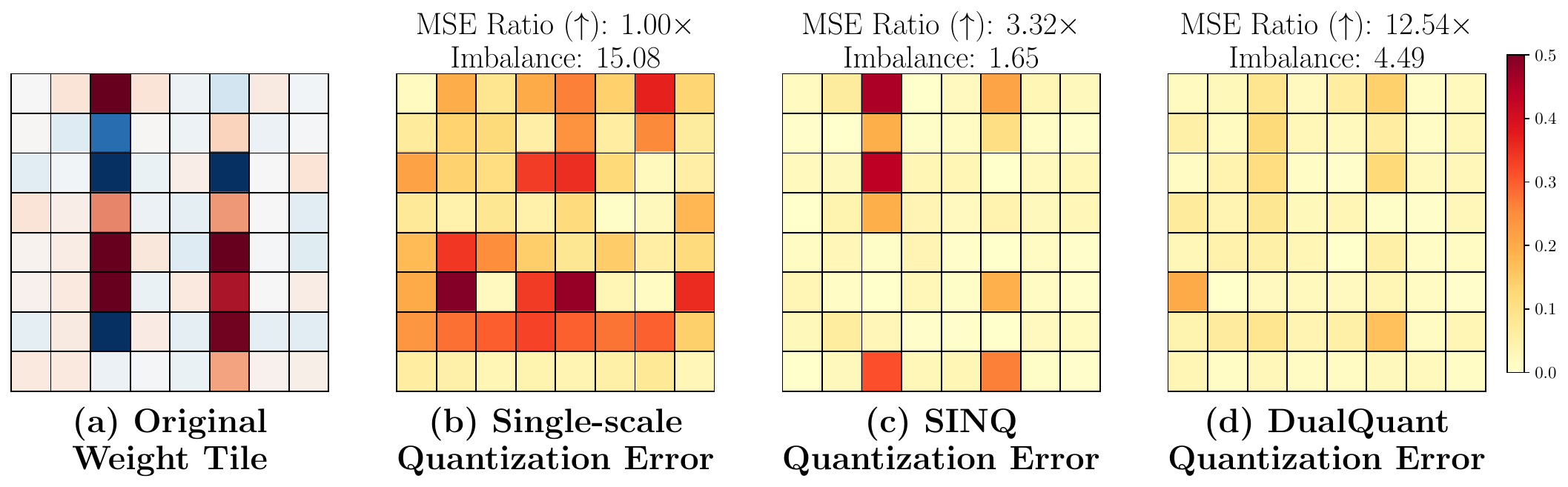}
    \caption{\textbf{Per-element quantization error $\scriptstyle |\mathbf{W} -
\hat{\mathbf{W}}|$ under different scale designs.} 
    Each panel reports the
MSE ratio ($\uparrow$) relative to single-scale and the
matrix imbalance (SINQ's objective). \mbox{\textbf{(a)} Original}
weight tile. \textbf{(b)} Single-scale: one outlier inflates
the entire row. \textbf{(c)} SINQ minimizes matrix imbalance
($1.65$) but achieves only $3.32\times$ MSE reduction.
\textbf{(d)} \method{} minimizes MSE, achieving
$12.54\times$ reduction, illustrating that low imbalance does not imply low
reconstruction error.
}
    \label{fig:motivation}
\end{figure}

\section{Quantization preliminaries}
\label{sec:preliminaries}
Quantization maps a tensor to a finite, format-specific set of representable values. Given a \mbox{tensor} $W$, let $\mathcal{G}$ denote the canonical, unit-scale grid of the target format, namely, the integers $\{-2^{B-1}, \dots, 2^{B-1}-1\}$ for INT$B$, or the
non-uniform exponent-mantissa grid for a floating-point format.
A quantizer $Q(\cdot)$ with scale $\Delta$ and zero-point $z$ maps
each element to the nearest representable value, $Q(W) = \Delta \cdot ( \mathrm{proj}_{\mathcal{G}}(W/\Delta + z) - z)$,
where $\mathrm{proj}_{\mathcal{G}}(\cdot)$ rounds and clips its argument to the nearest point on
$\mathcal{G}$, reducing to standard round-and-clip for integer formats and to nearest-level
projection on the non-uniform exponent-mantissa grid for floating-point formats.
Symmetric
quantization sets $z=0$ with $\Delta = \max(|W|)/q_{\max}$, where $q_{\max}$ denotes the maximum positive quantization level; asymmetric quantization uses nonzero
$z$ to span $[\min(W), \max(W)]$.
To reduce sensitivity to outliers, block-wise quantization partitions each row into groups of $G$ contiguous elements with independent scales. The OCP microscaling 
formats~\citep{rouhani2023microscaling} (MXFP4, MXINT4) use $G=32$ with a shared
E8M0 scale restricted to powers of two; NVFP4~\citep{nvidia2025nvfp4} uses $G=16$
with an FP8 UE4M3  scale. This motivates the format-agnostic design of \method{}, treating $Q(\cdot)$ as a black-box with an explicit scale constraint (Section~\ref{sec:dual_scale}).

\section{Methodology}
\label{sec:methodology}
\subsection{Problem description}
\label{sec:joint_objective}

For a linear layer computing $\mathbf{X}\mathbf{W}^\top$ with activation matrix $\mathbf{X} \in \mathbb{R}^{b \times n}$ and
weight matrix $\mathbf{W} \in \mathbb{R}^{m \times n}$, the joint W+A quantization error can be written as
\begin{equation}
\label{eq:joint_error_decomp}
\textstyle
    \mathbf{X}\mathbf{W}^\top - Q(\mathbf{X})\,Q(\mathbf{W})^\top
    =
    (\mathbf{X} - Q(\mathbf{X}))\mathbf{W}^\top
    +
    Q(\mathbf{X})(\mathbf{W} -
    Q(\mathbf{W}))^\top,
\end{equation}
where the activation error is amplified by $\mathbf{W}^\top$ and the weight error is amplified by $Q(\mathbf{X})$. Minimizing the two terms independently does not necessarily minimize the joint error, and minimizing~\eqref{eq:joint_error_decomp} directly is intractable without calibration data on $\mathbf{X}$. SmoothQuant~\citep{xiao2023smoothquant} targets this joint problem by collecting per-channel activation statistics from a calibration dataset to construct a per-channel rescaling applied prior to quantization, $\mathbf{X}\mathbf{W}^\top = (\mathbf{X}\,\mathrm{Diag}(\mathbf{s})^{-1}) (\mathbf{W}\, \mathrm{Diag}(\mathbf{s}))^\top$. We take a different approach: rather than collecting activation statistics from data, we show that the column scales obtained from minimizing the \emph{weight} reconstruction MSE on static weights implicitly perform a per-channel rescaling of the same form, without any access to $\mathbf{X}$, as developed next.

\subsection{Dual-scale parameterization and the activation
factorization}
\label{sec:dual_scale}

Standard block quantization assigns a single shared scale per block along each row of the weight matrix, so an outlier at position $(i,j)$ inflates the scale shared by every element in that block, while the column dimension remains entirely unaddressed. 
Figure~\ref{fig:motivation} illustrates this: the single-scale error matrix exhibits pronounced column-wise outlier patterns that persist across all rows. We mitigate this by introducing a second, column-wise scale alongside the per-block row scale. The resulting dual scaling distributes the error across both dimensions, substantially reducing the peak element-wise reconstruction error (Figure~\ref{fig:motivation}(d)). Specifically, we quantize $\mathbf{W} \in \mathbb{R}^{m \times n}$ (or a block thereof) as
\begin{equation}
\label{eq:dual_scale}
    \mathbf{W} \approx
    \mathrm{Diag}(\boldsymbol{\alpha})^{-1}\,
    Q\!\left(\mathrm{Diag}(\boldsymbol{\alpha})\,\mathbf{W}\,
    \mathrm{Diag}(\boldsymbol{\beta})\right)\,
    \mathrm{Diag}(\boldsymbol{\beta})^{-1}
    \;\triangleq\;
    \mathrm{Diag}(\boldsymbol{\alpha})^{-1}\,
    \hat{\mathbf{Q}}_{\boldsymbol{\alpha},\boldsymbol{\beta}}\,
    \mathrm{Diag}(\boldsymbol{\beta})^{-1},
\end{equation}
where $\boldsymbol{\alpha}$ and $\boldsymbol{\beta}$ are respectively row and column scales, $Q(\cdot)$ is a fixed format-specific quantizer, and $\hat{\mathbf{Q}}_{\boldsymbol{\alpha},\boldsymbol{\beta}} = Q\!\left(\mathrm{Diag}(\boldsymbol{\alpha})\,\mathbf{W}\,\mathrm{Diag}(\boldsymbol{\beta})\right)$ denotes the quantized representation stored at low precision. The inverse scales are constrained to their format-admissible sets: $\boldsymbol{\alpha}^{\odot -1} \in \mathbb{F}_\alpha^m$ and $\boldsymbol{\beta}^{\odot -1} \in \mathbb{F}_\beta^n$, where $(\cdot)^{\odot -1}$ denotes the element-wise inverse, and $\mathbb{F}_\alpha^m$, $\mathbb{F}_\beta^n$ are the admissible scale sets of the target format (BF16 for INT4; E8M0 for MXINT4/MXFP4; FP8 UE4M3 for NVFP4). The row scale absorbs row-level dynamic range as in standard formats, while the column scale captures the systematic column-wise structure that single-scale formats ignore. Hence, an outlier at position $(i,j)$ is now handled jointly by $\alpha_i$ and $\beta_j$, distributing the quantization
burden across both dimensions.

\paragraph{Column scales as activation rescalings.}

Substituting~\eqref{eq:dual_scale} into $\mathbf{X}\mathbf{W}^\top$ yields
\begin{equation}
\label{eq:activation_factorization}
\textstyle
    \mathbf{X}\mathbf{W}^\top
    \approx
    \left(\mathbf{X}\,
    \mathrm{Diag}(\boldsymbol{\beta})^{-1}\right)
    \hat{\mathbf{Q}}_{\boldsymbol{\alpha},\boldsymbol{\beta}}^\top\,
    \mathrm{Diag}(\boldsymbol{\alpha})^{-1}
    \;\triangleq\;
    \widetilde{\mathbf{X}}\,
    \widetilde{\mathbf{W}}^\top,
\end{equation}
where $\widetilde{\mathbf{X}} = \mathbf{X}\,\mathrm{Diag}(\boldsymbol{\beta})^{-1} \in \mathbb{R}^{b \times n}$ and
$\widetilde{\mathbf{W}}^\top =\hat{\mathbf{Q}}_{\boldsymbol{\alpha},\boldsymbol{\beta}}^\top\,\mathrm{Diag}(\boldsymbol{\alpha})^{-1} \in \mathbb{R}^{n \times m}$.
The quantity quantized at inference is not $X$ but the rescaled
activation $\widetilde{X}$, whose per-channel statistics are reshaped by $\boldsymbol{\beta}$. The column
scales therefore play two simultaneous roles: reduce weight quantization error
through~\eqref{eq:dual_scale} and rescale activations through~\eqref{eq:activation_factorization}.
This resembles the per-channel rescaling of SmoothQuant,
but with
reversed information flow: SmoothQuant derives its scales from activation statistics on a
calibration dataset, whereas $\boldsymbol{\beta}$ is obtained from the weight matrix alone.

\mahsa{This weight-derived rescaling benefits activation quantization; it is not merely an algebraic by-product. Intuitively, the output of a linear layer at position $(b,k)$ is $\sum_j X_{bj} W_{kj}$, so channels with large $\|W_{:,j}\|_2$ contribute disproportionately to the output, and these are empirically the channels where activation outliers concentrate~\citep{sun2024massive,dettmers2022llmint8,bondarenko2021understanding}. Since $\boldsymbol{\beta}$ is optimized to reduce weight quantization error, it attenuates these same high-magnitude channels using the weights alone. Its mechanism differs from calibration: rather than reproducing the per-channel statistics that methods such as SmoothQuant estimate from activation data, $\boldsymbol{\beta}$ has a direct, causal effect on activation quantization. 
Note that the correlation between $\beta$ and such statistics is heterogeneous and projection-dependent\green{, as shown in} Table~\ref{tab:beta_sq_corr} in Section~\ref{sec:ablation}. With the quantized weights held fixed, applying $\boldsymbol{\beta}$ to the activations measurably reduces their quantization error, without any calibration data (Appendix~\ref{app:beta_causal} isolates this by measuring per-projection activation quantization error with and without $\boldsymbol{\beta}$ under fixed quantized weights). Because this rescaling is applied inside the activation-quantization kernel, its runtime cost is negligible (Section~\ref{sec:overhead}).}

\begin{algorithm}
\footnotesize
\caption{\method{}: Dual-Scale Quantization via Coordinate
Descent}
\label{alg:dual_scale}
\begin{algorithmic}[1]
\Require Weight tile $\mathbf{W}$,
quantizer $Q(\cdot)$, admissible sets $\mathbb{F}_\alpha^m$,
$\mathbb{F}_\beta^n$, number of iterations $T$,
initialization $(\boldsymbol{\alpha}_0, \boldsymbol{\beta}_0)$
\State $\boldsymbol{\alpha} \leftarrow \boldsymbol{\alpha}_0$,
\quad $\boldsymbol{\beta} \leftarrow \boldsymbol{\beta}_0$
\For{$t = 1, \ldots, T$}
    \State $\hat{\mathbf{Q}} \leftarrow
    Q\!\left(\mathrm{Diag}(\boldsymbol{\alpha})\,\mathbf{W}\,
    \mathrm{Diag}(\boldsymbol{\beta})\right)$
    \State $\mathbf{R} \leftarrow
    \hat{\mathbf{Q}}\,\mathrm{Diag}(\boldsymbol{\beta})^{-1}$
    \State $\boldsymbol{\alpha}_{\mathrm{out}}^{\odot -1} \leftarrow
    \Pi_{\mathbb{F}_\alpha^m}\!\left(
    (\mathbf{W} \odot \mathbf{R})\,\mathbf{1}_n\;\oslash\; \|\mathbf{R}\|_{2,\mathrm{row}}^2 \right)$ and $\alpha \leftarrow 1 \oslash \alpha_{\text{out}}^{\odot -1}$ 
    \Comment{Row scale update + project} 
    \State $\hat{\mathbf{Q}} \leftarrow
    Q\!\left(\mathrm{Diag}(\boldsymbol{\alpha})\,\mathbf{W}\,
    \mathrm{Diag}(\boldsymbol{\beta})\right)$
    \State $\mathbf{S} \leftarrow
    \mathrm{Diag}(\boldsymbol{\alpha})^{-1}\hat{\mathbf{Q}}$
    \State $\boldsymbol{\beta}_{\mathrm{out}}^{\odot -1} \leftarrow
\Pi_{\mathbb{F}_\beta^n}\!\left(
(\mathbf{W} \odot \mathbf{S})^\top\mathbf{1}_m \;\oslash\;  \|\mathbf{S}\|_{2,\mathrm{col}}^2 \right)$ and $\beta \leftarrow 1 \oslash \beta_{\text{out}}^{\odot -1}$
\Comment{Col scale update + project}
\EndFor
\State \Return $\boldsymbol{\alpha}_{\mathrm{out}}^{\odot -1} $,\; $\boldsymbol{\beta}_{\mathrm{out}}^{\odot -1}$,\;
$Q\!\left(\mathrm{Diag}(\boldsymbol{\alpha})\,\mathbf{W}\,
\mathrm{Diag}(\boldsymbol{\beta})\right)$
\end{algorithmic}
\vspace{1pt}
\hrule
\vspace{1pt}
{\footnotesize \textbf{Notation:}
$\|\cdot\|_{2,\mathrm{row}}^2$ and
$\|\cdot\|_{2,\mathrm{col}}^2$: squared $\ell_2$-norm
along each row and column;
$\oslash$: element-wise division;
$\odot$: element-wise multiplication;
$(\cdot)^{\odot -1}$: element-wise inverse;
$\Pi_{\mathbb{F}}$: projection to nearest value in $\mathbb{F}$.}
\end{algorithm}

\subsection{MSE optimization via block coordinate descent}
\label{sec:bcd}
We optimize the row and column scales by minimizing the weight reconstruction MSE:
\begin{equation}
\label{eq:primal}
    \textstyle
    \min_{\boldsymbol{\alpha},\, \boldsymbol{\beta}}\;\;
    \left\|\mathbf{W} -
    \mathrm{Diag}(\boldsymbol{\alpha})^{-1}\,
    \hat{\mathbf{Q}}_{\boldsymbol{\alpha},\boldsymbol{\beta}}\,
    \mathrm{Diag}(\boldsymbol{\beta})^{-1}
    \right\|_F^2
    \quad \text{s.t.} \quad
    \boldsymbol{\alpha}^{\odot -1} \in \mathbb{F}_\alpha^m,\;\;
    \boldsymbol{\beta}^{\odot -1} \in \mathbb{F}_\beta^n.
\end{equation}
Note that the scales determine which grid points the quantizer assigns when forming $\hat{\mathbf{Q}}_{\boldsymbol{\alpha},\boldsymbol{\beta}}$ and simultaneously rescale the quantized output for reconstruction. Since grid-point selection is piecewise-constant, this joint dependence is \emph{non-differentiable}.

\paragraph{Variable splitting.}
To tackle~\eqref{eq:primal}, we duplicate each scale into inner copies controlling grid assignment and outer copies controlling reconstruction~\citep{boyd2011admm}, yielding
\begin{equation}
\label{eq:split}
\begin{aligned}    \min_{\substack{\boldsymbol{\alpha}_{\mathrm{in}},\,\boldsymbol{\beta}_{\mathrm{in}},\boldsymbol{\alpha}_{\mathrm{out}},\,\boldsymbol{\beta}_{\mathrm{out}}}}\;\;&
    \left\|\mathbf{W} -
    \mathrm{Diag}(\boldsymbol{\alpha}_{\mathrm{out}})^{-1}\,
    \hat{\mathbf{Q}}_{\boldsymbol{\alpha}_{\mathrm{in}},\boldsymbol{\beta}_{\mathrm{in}}}\,
    \mathrm{Diag}(\boldsymbol{\beta}_{\mathrm{out}})^{-1}
    \right\|_F^2 \\
    \text{s.t.} \quad \;\quad \quad &
    \boldsymbol{\alpha}_{\mathrm{in}} =
    \boldsymbol{\alpha}_{\mathrm{out}},\;\;
    \boldsymbol{\beta}_{\mathrm{in}} =
    \boldsymbol{\beta}_{\mathrm{out}},\;\;
    \boldsymbol{\alpha}_{\mathrm{out}}^{\odot -1} \in
    \mathbb{F}_\alpha^m,\;\;
    \boldsymbol{\beta}_{\mathrm{out}}^{\odot -1} \in
    \mathbb{F}_\beta^n.
\end{aligned}
\end{equation}
With the inner copies fixed, $\hat{\mathbf{Q}}_{\boldsymbol{\alpha}_{\mathrm{in}},\boldsymbol{\beta}_{\mathrm{in}}}$ is a constant matrix and the objective becomes bilinear in $\boldsymbol{\alpha}_{\mathrm{out}}^{\odot -1}$ and $\boldsymbol{\beta}_{\mathrm{out}}^{\odot -1}$, decoupling into $m$ independent scalar subproblems over $\boldsymbol{\alpha}$ and $n$ over $\boldsymbol{\beta}$. A block coordinate descent algorithm is then used~\citep{tseng2001convergence, razaviyayn2013unified} to alternate between updating the outer scales (in closed form) and re-synchronizing the inner copies (Appendix~\ref{app:bcd}).

\paragraph{Closed-form scale updates.}
Hereafter, we write $\hat{\mathbf{Q}}$ for $\hat{\mathbf{Q}}_{\boldsymbol{\alpha}_{\mathrm{in}},\boldsymbol{\beta}_{\mathrm{in}}}$ for brevity. Defining $\mathbf{R}_{i,:} = \hat{\mathbf{Q}}_{i,:}\, \mathrm{Diag}(\boldsymbol{\beta}_{\mathrm{out}})^{-1}$ and $\mathbf{S}_{:,j} =
\mathrm{Diag}(\boldsymbol{\alpha}_{\mathrm{out}})^{-1} \hat{\mathbf{Q}}_{:,j}$, the optimal outer scales are (Appendix~\ref{app:closed_form}): $\alpha_i^* =
    \tfrac{\|\mathbf{R}_{i,:}\|_2^2}
    {\mathbf{W}_{i,:}\,\mathbf{R}_{i,:}^\top}$, and $
    \beta_j^* =
    \tfrac{\|\mathbf{S}_{:,j}\|_2^2}
    {\mathbf{S}_{:,j}^\top\,\mathbf{W}_{:,j}}$.
Both updates have the same structure: the optimal scale is the ratio of the energy of the rescaled quantized slice to its inner product with the original weight slice. After each update, the scale is projected onto the admissible set by rounding $\alpha_i^{-1}$ (resp.\ $\beta_j^{-1}$) to the nearest value in $\mathbb{F}_\alpha^m$ (resp.\ $\mathbb{F}_\beta^n$). The BCD
scheme cycles through: (i)~update $\boldsymbol{\alpha}_{\mathrm{out}}$, (ii)~enforce consistency and re-quantize, (iii)~update $\boldsymbol{\beta}_{\mathrm{out}}$, (iv)~enforce consistency and re-quantize, as summarized in Algorithm~\ref{alg:dual_scale}.

\paragraph{Convergence analysis.}
We establish convergence under the following mild assumption.
\begin{assumption}[Non-degenerate quantization]
\label{ass:nondegenerate}
For all iterations $t$ and rows $i$, no row of
$\mathrm{Diag}(\boldsymbol{\alpha}^{(t)})\mathbf{W}
\mathrm{Diag}(\boldsymbol{\beta}^{(t)})$ is mapped entirely
to zero by $Q(\cdot)$. This holds generically since LLM weight
rows follow near-Gaussian distributions with non-negligible
$\ell_2$ norm~\citep{dettmers2023qlora}. In the
degenerate case $\hat{\mathbf{Q}}_{i,:} = \mathbf{0}$, the scale
$\alpha_i$ does not affect the objective and is set to $1$ by convention.

\end{assumption}

\begin{theorem}[Convergence of \method{}; proof in Appendix~\ref{app:proof_theory1}]
\label{thm:convergence}
Let Assumption~\ref{ass:nondegenerate} hold, and let $\{(\boldsymbol{\alpha}^{(t)},
\boldsymbol{\beta}^{(t)})\}_{t \geq 1}$ be the sequence generated
by Algorithm~\ref{alg:dual_scale}. Suppose $\hat{\mathbf{Q}}^{(t)}$
stabilizes at some finite iteration $T$, i.e.,
$\hat{\mathbf{Q}}^{(t)} = \hat{\mathbf{Q}}^{(T)}$ for all
$t \geq T$. Then:

\smallskip
\noindent\textup{(i)}~\textbf{Descent guarantee.}
$\mathrm{MSE}_{\mathbf{W}}(\boldsymbol{\alpha}^{(t+1)},
\boldsymbol{\beta}^{(t+1)}) \leq
\mathrm{MSE}_{\mathbf{W}}(\boldsymbol{\alpha}^{(t)},
\boldsymbol{\beta}^{(t)})$ for all $t \geq T$.

\smallskip
\noindent\textup{(ii)}~\textbf{Stationarity.}
Every limit point $(\boldsymbol{\alpha}^*, \boldsymbol{\beta}^*)$
is a coordinate-wise stationary point of~\eqref{eq:primal},
satisfying
$\boldsymbol{\alpha}^* =
\arg\min_{\boldsymbol{\alpha}}\,\mathrm{MSE}_{\mathbf{W}}
(\boldsymbol{\alpha}, \boldsymbol{\beta}^*)$ and
$\boldsymbol{\beta}^* =
\arg\min_{\boldsymbol{\beta}}\,\mathrm{MSE}_{\mathbf{W}}
(\boldsymbol{\alpha}^*, \boldsymbol{\beta})$.
\end{theorem}
\noindent The stabilization condition on $\hat{\mathbf{Q}}$ is empirically 
\mahsa{supported by} Figure~\ref{fig:qhat_stable}(a), where its relative change decays to near zero within approximately $T{\approx}15$ iterations for all layers and blocks. Furthermore, Figure~\ref{fig:qhat_stable}(b) confirms Theorem~\ref{thm:convergence}(i): the MSE decreases monotonically and plateaus within a finite number of iterations.

\begin{figure}
\centering
\includegraphics[width=0.73\linewidth]{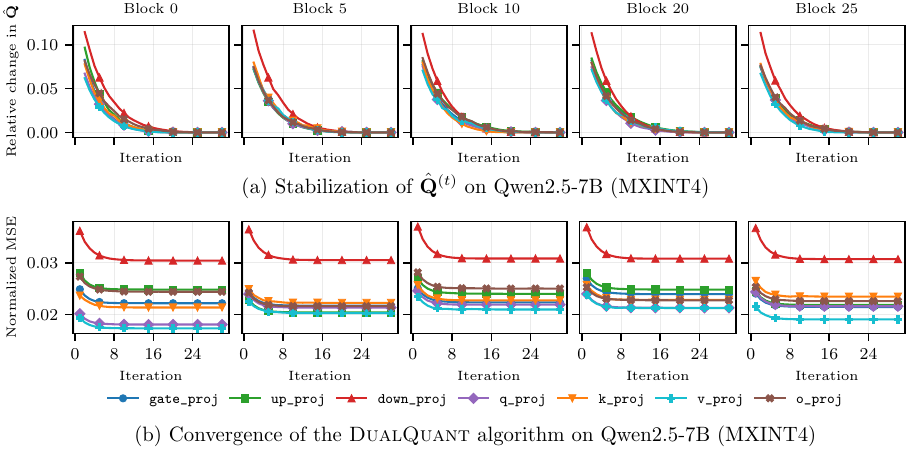}
\caption{\textbf{(a)} Relative change $\scriptstyle \|\hat{\mathbf{Q}}^{(t)} -
\hat{\mathbf{Q}}^{(t-1)}\|_F / \|\hat{\mathbf{Q}}^{(t-1)}\|_F$
vs.\ BCD iteration for five transformer blocks.
The change decays to near zero by $t{\approx}16$ for all layers and blocks,
validating the stabilization condition of Theorem~\ref{thm:convergence}.
\textbf{(b)} Normalized MSE
$\scriptstyle \|\mathbf{W} - \hat{\mathbf{W}}\|_F^2 / \|\mathbf{W}\|_F^2$
vs.\ BCD iteration for the same blocks.
All curves decrease monotonically and plateau within ${\approx}10$ iterations,
confirming Theorem~\ref{thm:convergence}(i).
}
\label{fig:qhat_stable}
\end{figure}

\subsection{Inference-time procedure and overhead}
\label{sec:inference}

The offline quantization procedure yields three artifacts per
linear layer: the quantized weight matrix $\hat{\mathbf{Q}}$,
row scales $\boldsymbol{\alpha}$, and column scales
$\boldsymbol{\beta}$. At inference, the layer computes
$Q(\widetilde{\mathbf{X}})\,\hat{\mathbf{Q}}^\top\,\mathrm{Diag}(\boldsymbol{\alpha})^{-1}$,
where $\widetilde{\mathbf{X}} = \mathbf{X}\,\mathrm{Diag}(\boldsymbol{\beta})^{-1}$,
and $Q(\widetilde{\mathbf{X}})$ denotes the format-native
row-wise quantization of the rescaled activations. The row
scale $\boldsymbol{\alpha}$ is absorbed into the per-row
dequantization scale of $\hat{\mathbf{Q}}$ at quantization
time and incurs no runtime cost.
The column scale $\boldsymbol{\beta}$ introduces a per-channel
rescaling of the activations \emph{before} quantization. It
does not replace the format-native activation scale: $Q(\cdot)$
still computes its own per-row scale on
$\widetilde{\mathbf{X}}$ as usual. The two operate at different
granularities: $\boldsymbol{\beta}$ attenuates outlier
channels globally, while the format-native scale captures the
residual per-token dynamic range.
Unlike SmoothQuant
and
AWQ,
which share a single scale across projections and fold it offline,
\method{}
optimizes $\boldsymbol{\beta}$ independently per layer, which
allows each layer to reach its own optimum.
The resulting per-channel multiply 
is applied inside the activation-quantization kernel without a separate kernel launch (Section~\ref{sec:overhead}).

\section{Experiments} 
\label{sec:experiments}

\paragraph{Models and formats.}
We evaluate two model families: Qwen (Qwen2.5-7B, Qwen3-0.6B, Qwen3-14B) and Llama (Llama-2-7B, Llama-3.1-8B, Llama-3.2-1B), covering small, medium, and large scales across two architectures. We target 4-bit quantization formats spanning the major families of modern low-bit deployment: INT4 (block size $G=64$), MXINT4 ($G=32$), MXFP4 ($G=32$), and NVFP4 ($G=16$). INT4 is the standard integer baseline~\citep{frantar2022gptq, lin2024awq}. For all formats, the obtained scales are constrained to $\mathbb{F}$ as described in Section~\ref{sec:bcd}. The corresponding licenses for all models are included in Appendix~\ref{app:licenses}.

\paragraph{Evaluation setup.}
We evaluate \method{} across a range of tasks capturing different aspects of model performance: \emph{language modeling ability} via perplexity on WikiText-2~\citep{merity2017pointer} and C4~\citep{raffel2020c4} (test splits), \emph{commonsense reasoning} and \emph{language understanding} via normalized zero-shot accuracy on MMLU~\citep{hendrycks2021mmlu}, HellaSwag~\citep{zellers2019hellaswag}, PIQA~\citep{bisk2020piqa}, ARC-Easy~\citep{clark2018think}, and \green{WinoGrande}~\citep{sakaguchi2021winogrande}.
The corresponding licenses 
are included in Appendix~\ref{app:licenses}.
All methods requiring calibration
data (AWQ, SmoothQuant) use 256 randomly sampled sequences
of length 8192 from the WikiText-2 training split; calibration and
evaluation sets are strictly disjoint. We quantize all linear
projections in every transformer block;
\texttt{lm\_head} is kept in full
precision. All fake-quantization experiments are conducted on NVIDIA GH200 80~GB GPUs, with weights dequantized to FP16/BF16 during inference. Inference efficiency is evaluated separately using the real packed W4A4 implementation on a single NVIDIA RTX 4090 with 24~GB of memory (Section~\ref{sec:overhead}). Weight-only results and 8-bit W+A
results are provided in Appendix~\ref{app:weight_only} and
Appendix~\ref{app:8bit_wa}, respectively. 

\subsection{Joint weight-and-activation quantization}
\label{sec:wa_quant}
Tables~\ref{tab:ppl_wa_llama} (Llama family, main text)
and~\ref{tab:ppl_wa_qwen} (Qwen family,
Appendix~\ref{app:additional_results}) compare \method{} against five baselines: \textbf{RTN}, which applies round-to-nearest quantization to both weights and activations; \textbf{AWQ+}, which extends AWQ~\citep{lin2024awq} with RTN activation quantization; \textbf{SmoothQuant}~\citep{xiao2023smoothquant}, which uses calibration-based per-channel activation rescaling; \textbf{SINQ+}, our extension of
SINQ~\citep{muller2025sinq}
(see Appendix~\ref{app:sinq_extension}); and \textbf{QuaRot} \citep{ashkboos2024quarot}, a calibration-free approach that applies rotation matrices to weights and activations to suppress outliers prior to quantization, followed by RTN activation quantization.
Additional details on baselines are provided in Appendix~\ref{app:baselines}. 

Table~\ref{tab:ppl_wa_llama} reports joint W+A quantization
perplexity on the Llama family across four 4-bit formats.
\method{} consistently matches or surpasses the 
calibration-based 
SmoothQuant, while requiring no calibration data. On MXFP4, \method{} improves over SmoothQuant on WikiText-2 for all three models, substantially on Llama-2-7B and Llama-3.2-1B and narrowly on Llama-3.1-8B; on C4 the same holds except for Llama-3.1-8B, where \method{} trails slightly. \green{On MXINT4, \method{} is better than SmoothQuant across both WikiText-2 and C4 for all three models}. On INT4, \method{} surpasses both SmoothQuant and
AWQ+ across all models, demonstrating that weight-derived
scales provide effective activation rescaling even on the
standard integer format. On NVFP4, \method{} achieves the
strongest results among all methods on Llama-3.1-8B and Llama-3.2-1B, with the same overall trend on C4.
In contrast, both calibration-free baselines degrade substantially on MX and NV
formats: SINQ+ degrades on MXINT4 and MXFP4 
and diverges completely on NVFP4, as its format-unaware scales fall outside the FP8 representable range; QuaRot similarly diverges on NVFP4 and degrades substantially on MX formats,
as its Hadamard rotation redistributes activation magnitude in a way
incompatible with block floating-point scaling. 
\method{} avoids both failure modes by
design, incorporating format constraints directly into the optimization and requiring no
rotation or calibration data. Corresponding results on the Qwen family
(Table~\ref{tab:ppl_wa_qwen}, Appendix~\ref{app:additional_results})
show consistent trends\green{.}%
\footnote{The \method{}--SINQ+ gap persists without our extensions: in SINQ's native weight-only INT4 setting, \method{} is more accurate on three of four models (Appendix~\ref{app:weight_only}), so it stems from the \emph{objective} itself.}

\input{tables/main_table_llama_family}
\label{sec:downstream}
Table~\ref{tab:downstream} reports zero-shot accuracy on MMLU,
HellaSwag,
PIQA,
ARC-Easy,
and \green{WinoGrande}
for Llama-3.1-8B across all four W+A quantization formats. Corresponding results for the \green{Qwen3-14B} are reported in Table~\ref{tab:downstream_qwen} in Appendix~\ref{app:additional_results}. \method{} achieves the best calibration-free accuracy across nearly all settings. On NVFP4, 
\method{} surpasses every baseline, including the calibration-based SmoothQuant, across all five tasks, while both SINQ+ and QuaRot diverge completely on
NVFP4, making \method{} \green{by far the strongest calibration-free method on this format}. On MXINT4 and MXFP4, \method{} substantially outperforms both calibration-free
baselines \green{on most tasks}: SINQ+ collapses toward random-guess level, while QuaRot also degrades significantly. On INT4, where all methods remain
viable, \method{} achieves the best calibration-free accuracy on MMLU, HellaSwag, and
ARC-Easy.
\input{tables/downstream_main}

\subsection{Runtime and memory analysis}
\label{sec:overhead}

\paragraph{Offline quantization cost.}
The optimization that produces $(\boldsymbol{\alpha},
\boldsymbol{\beta})$ operates entirely on the static weight
matrix with no calibration forward passes, no Hessian, and no
inner optimization loop. Each coordinate update is a single
closed-form expression, making the offline cost comparable to
RTN and orders of magnitude cheaper than GPTQ or AWQ. 

\paragraph{Real packed W4A4 deployment.}
To verify that fake-quantization results translate to actual deployment, we implement a packed W4A4 inference path for Llama-3.1-8B~\citep{dubey2024llama3} on a single NVIDIA RTX 4090~\citep{nvidia_rtx4090}. The implementation uses INT4 with group size $G{=}64$ and is built on the Nunchaku/SVDQuant~\citep{li2025svdquant} W4A4 GEMM backend (Appendix~\ref{app:packed-w4a4}). Accuracy, memory usage, and latency are measured directly on the packed deployment path \emph{without} fake quantization. WikiText-2 perplexity is $6.74$, closely matching the fake-quantization result of $6.73$ (Table~\ref{tab:ppl_wa_llama}) with an absolute difference of $0.01$. This agreement confirms that the fake-quantization results in Section~\ref{sec:experiments} accurately reflect the deployed model.

\paragraph{Inference and memory overhead.}
The row scale $\alpha$ is folded into the native group scale and introduces no runtime overhead. The column scale $\beta$ is incorporated into the activation-quantization kernel without an additional kernel launch, while weights remain packed throughout inference. Thus, activation quantization, including the application of $\beta$, constitutes the only per-step quantization overhead, accounting for $4.3$\texttt{-}$6.1\%$ of total W4A4 GPU time (Table~\ref{tab:packed-latency}). Despite this overhead, \green{\method{}} W4A4 achieves $1.39$\texttt{-}$2.46\times$ end-to-end speedup over BF16 across prefill and decode at batch sizes $1$ and $16$ (Table~\ref{tab:packed-w4a4}). For Llama-3.1-8B relative to BF16, the packed deployment reduces the weight footprint from $16.20$~GB to $5.92$~GB and end-to-end memory usage by $2.51$\texttt{-}$2.61\times$ (Table~\ref{tab:packed-w4a4}). Component-level timing and memory breakdowns are provided in Appendix~\ref{app:packed-w4a4} (Table~\ref{tab:packed-memory}).

\begin{table}
\centering
\scriptsize
\caption{End-to-end memory usage and throughput of the real packed W4A4 deployment versus BF16 for Llama-3.1-8B on an RTX 4090 with INT4, $G{=}64$. Prefill uses an 82-token prompt without a KV cache. Decode generates $100$ tokens with a 182-token KV cache. Memory includes weights, scales, unquantized embedding and lm\_head, activations, and the KV cache for decode.}
\label{tab:packed-w4a4}
\resizebox{0.8\textwidth}{!}{
\begin{tabular}{lcccccc}
\toprule
\multirow{2}{*}{\textbf{Workload}} & \textbf{Memory (GB)} & \hl \textbf{Memory (GB)} & \hl \textbf{Memory} &
\textbf{Throughput (tok/s)} & \hl \textbf{Throughput (tok/s)} & \hl \\
 & \textbf{BF16}  & \hl \textbf{W4A4}  & \hl \textbf{Reduction} &
\textbf{BF16}  & \hl \textbf{W4A4}  & \multirow{-2}{*}{\hl \textbf{Speedup}} \\
\midrule
Prefill $(b{=}1)$  & 16.33 & \hl6.25 & \hl $2.61\times$ & 3904.6  & \hl 6421.4  & \hl $1.64\times$ \\
Prefill $(b{=}16)$ & 16.66 & \hl6.58 & \hl $2.53\times$ & 8940.3  & \hl 22003.4 & \hl $2.46\times$ \\
Decode $(b{=}1)$   & 16.34 & \hl6.26 & \hl $2.61\times$ & 56.6    & \hl 83.7    & \hl $1.48\times$ \\
Decode $(b{=}16)$  & 16.75 & \hl6.67 & \hl $2.51\times$ & 694.5   & \hl 964.1   & \hl $1.39\times$ \\
\bottomrule
\end{tabular}}
\end{table}

\begin{table}
\centering
\renewcommand{\arraystretch}{0.95}
\begin{minipage}[t]{0.48\textwidth}
\centering
\caption{Scale ablation for W+A quantization (WikiText-2 PPL $\downarrow$).
 \textbf{Bold}: best per
column. 
$\boldsymbol{\alpha}$-only fixes
$\boldsymbol{\beta}{=}\mathbf{1}_n$;
$\boldsymbol{\beta}$-only fixes
$\boldsymbol{\alpha}{=}\mathbf{1}_m$.}
\label{tab:ablation_scales}
\vspace{2pt}
\scriptsize
\setlength{\tabcolsep}{3pt}
\begin{tabular}{@{}lcccc@{}}
\toprule
\multirow{2}{*}[-2pt]{\textbf{Variant}} &
\multicolumn{2}{c}{\textbf{Qwen2.5-7B}} &
\multicolumn{2}{c}{\textbf{Llama-3.1-8B}} \\
\cmidrule(lr){2-3}\cmidrule(lr){4-5}
 & \textbf{MXFP4} & \textbf{NVFP4}
 & \textbf{MXFP4} & \textbf{NVFP4} \\
\midrule
RTN (no scales)            & 7.78 & 7.07 & 7.60 & 9.04 \\
$\boldsymbol{\alpha}$-only & 7.78 & 7.05 & 7.60 & 8.27 \\
$\boldsymbol{\beta}$-only  & 7.58 & 6.88 & 7.46 & 7.79 \\
\hl Both ($\boldsymbol{\alpha}+\boldsymbol{\beta}$) &
\hl\textbf{7.45} & \hl\textbf{6.53} & \hl\textbf{7.35} & \hl\textbf{6.83} \\
\bottomrule
\end{tabular}%
\end{minipage}
\hfill
\begin{minipage}[t]{0.48\textwidth}
\centering
\caption{Per-layer Pearson correlation  between \method{}'s $\beta$ and
SmoothQuant's calibrated scale $s$,  on
    Llama-3.1-8B.}
\label{tab:beta_sq_corr}
\vspace{2pt}
\scriptsize
\setlength{\tabcolsep}{4pt}
\begin{tabular}{@{}lccc@{}}
\toprule
\textbf{Projection} & \textbf{$\lambda{=}0$} & \textbf{$\lambda{=}0.5$} & \textbf{$\lambda{=}1.0$} \\
\midrule
\texttt{q\_proj}    & +0.96 & +0.09 & $-0.41$ \\
\texttt{k\_proj}    & +0.97 & +0.06 & $-0.57$ \\
\texttt{v\_proj}    & +0.93 & +0.53 & +0.07  \\
\texttt{gate\_proj} & +0.91 & +0.55 & $-0.34$ \\
\texttt{up\_proj}   & +0.83 & +0.41 & $-0.28$ \\
\bottomrule
\end{tabular}%
\end{minipage}%
\end{table}

\subsection{Ablation studies}
\label{sec:ablation}
\paragraph{Row scales, column scales, or both.}
Table~\ref{tab:ablation_scales} compares three variants of
\method{} 
on MXFP4 and NVFP4 
W+A quantization. The \textbf{$\boldsymbol{\alpha}$-only} variant
improves over RTN primarily on NVFP4, whose finer FP8 scale
grid offers more degrees of freedom than MXFP4's coarser E8M0 grid (e.g.,
$9.04 \to 8.27$ on Llama-3.1-8B), while the \textbf{$\boldsymbol{\beta}$-only}
variant yields consistently larger gains by simultaneously
reducing column-wise weight quantization error and rescaling activations via the factorization of
Section~\ref{sec:dual_scale}.
Combining both yields the best results in all cases with roughly additive
improvement, confirming that row and column scales target
complementary sources of error.

\paragraph{Hyperparameter selection.}
\method{} has two hyperparameters: BCD iteration, $T$, and
initialization $(\boldsymbol{\alpha}_0, \boldsymbol{\beta}_0)$.
The monotone-descent
objective guarantees convergence
from any starting point, but a good initialization reduces the
iterations needed. We initialize $\alpha_{0,i} =
\max_j |W_{ij}|$ (format-native absmax) and
$\beta_{0,j} = \sum_i |W_{ij}|$ (per-column
$\ell_1$-norm). Considering Figure~\ref{fig:qhat_stable}, 
we use $T{=}15$, which we found sufficient across the evaluated settings.

\paragraph{Correlation between the scales of \method{} and SmoothQuant.}
To relate \method{}'s weight-derived rescaling $\beta$ to SmoothQuant's calibration-derived scale, we compute the per-layer Pearson correlation~\citep{pearson1895regression} across $\lambda \in \{0, 0.5, 1.0\}$, the migration strength in SmoothQuant's scale equation, on Llama-3.1-8B (Table~\ref{tab:beta_sq_corr}). At $\lambda=0$ (purely
weight-derived $s$), $\beta$ aligns near-perfectly with SmoothQuant's scale; correlation remains
positive across most projections at $\lambda=0.5$ and decreases monotonically toward $\lambda=1.0$
(activation-dominated $s$), as expected for a weight-only method. Correlation is descriptive and does not imply equivalent quantization behavior.


\section{Conclusion}
\label{sec:conclusion}

We present \method{}, a calibration-free PTQ framework that jointly quantizes weights and activations by minimizing weight-reconstruction MSE via BCD with closed-form updates and provable convergence, operating across INT4, the MX family, and NVFP4. Its weight-driven column scales double as activation rescalings, so on Llama and Qwen \method{} matches or surpasses calibration-based SmoothQuant and clearly outperforms calibration-free QuaRot and SINQ+, without any calibration data. A packed W4A4 path for Llama-3.1-8B further delivers up to 2.61$\times$ lower memory use and 2.46$\times$ higher throughput than BF16 on an RTX 4090.

\pagebreak

\bibliography{iclr2027_conference}
\bibliographystyle{iclr2027_conference}

\appendix
\input{appendix/optimization_bcd}

\input{appendix/closed_form_solution}
\input{appendix/theory1_proof}

\input{appendix/sinq}
\input{appendix/causal_effect_of_beta}
\input{appendix/baseline_methods}

\input{appendix/additional_reults}

\input{appendix/weight_only_results}

\input{appendix/packed_w4a4}

\section{Licenses for existing assets} 
\label{app:licenses}
All models, datasets, and software libraries used in this work are publicly available and used in accordance with their respective licenses. Table~\ref{tab:licenses} summarizes the assets along with their licenses and sources. The Llama models are released under Meta's Llama Community Licenses and the Qwen models under Apache 2.0.
The evaluation datasets are released under a mix of permissive and share-alike licenses: WikiText-2 (CC BY-SA 3.0), C4 (ODC-BY 1.0), HellaSwag (MIT), MMLU (MIT), PIQA (AFL-3.0), ARC-Easy (CC BY-SA 4.0), and WinoGrande (CC-BY).
All assets were used solely for academic research purposes consistent with their terms of use.
\input{tables/licenses_table}

\end{document}

%% file: math_commands.tex
\usepackage{amsmath,amsfonts,bm}

\def\eqref#1{equation~\ref{#1}}

\def\1{\bm{1}}

\def\ra{{\textnormal{a}}}

\def\rx{{\textnormal{x}}}

\def\rva{{\mathbf{a}}}

\def\erva{{\textnormal{a}}}

\def\ervx{{\textnormal{x}}}

\def\rmA{{\mathbf{A}}}

\def\vmu{{\bm{\mu}}}
\def\vtheta{{\bm{\theta}}}
\def\va{{\bm{a}}}

\def\ve{{\bm{e}}}

\def\vx{{\bm{x}}}

\def\eva{{a}}

\def\mA{{\bm{A}}}

\def\mH{{\bm{H}}}
\def\mI{{\bm{I}}}
\def\mJ{{\bm{J}}}

\def\mX{{\bm{X}}}

\def\mSigma{{\bm{\Sigma}}}

\DeclareMathAlphabet{\mathsfit}{\encodingdefault}{\sfdefault}{m}{sl}
\SetMathAlphabet{\mathsfit}{bold}{\encodingdefault}{\sfdefault}{bx}{n}
\newcommand{\tens}[1]{\bm{\mathsfit{#1}}}
\def\tA{{\tens{A}}}

\def\tX{{\tens{X}}}

\def\gG{{\mathcal{G}}}

\def\sA{{\mathbb{A}}}
\def\sB{{\mathbb{B}}}

\def\sS{{\mathbb{S}}}

\def\emA{{A}}

\newcommand{\etens}[1]{\mathsfit{#1}}

\def\etA{{\etens{A}}}

\newcommand{\E}{\mathbb{E}}

\newcommand{\R}{\mathbb{R}}

\newcommand{\KL}{D_{\mathrm{KL}}}
\newcommand{\Var}{\mathrm{Var}}

\newcommand{\Cov}{\mathrm{Cov}}

\newcommand{\normltwo}{L^2}
\newcommand{\normlp}{L^p}

\newcommand{\parents}{Pa} 

%% file: tables/main_table_llama_family.tex
\begin{table}[t]
\centering
\caption{WikiText-2 and C4 perplexity ($\downarrow$) for joint
W+A quantization at 4-bit precision on the Llama family. \method{} applies $\boldsymbol{\beta}$-rescaling to activations without calibration data. SmoothQuant and AWQ+ use activation statistics from 256 calibration samples. SINQ+ is our extension of SINQ to \green{W+A}. QuaRot rotates activations/weights online with a Hadamard transform (no calibration). 
\textbf{Bold}: best calibration-free (\textbf{CF}) result. ($\dagger$): degenerate result $(PPL>1000)$; method incompatible with format.
\cmark: calibration-free. \xmark: requires calibration data.}
\label{tab:ppl_wa_llama}
\setlength{\tabcolsep}{3pt}
\renewcommand{\arraystretch}{0.9}
\resizebox{0.85\textwidth}{!}
{%
\begin{tabular}{@{}cllc cccccc@{}}
\toprule
\multirow{2}{*}[-2pt]{\textbf{Bits}} & \multirow{2}{*}[-2pt]{\textbf{Format}} & \multirow{2}{*}[-2pt]{\textbf{Method}} & \multirow{2}{*}[-2pt]{\textbf{CF}} &
\multicolumn{3}{c}{\textbf{WikiText-2 PPL} $\downarrow$} &
\multicolumn{3}{c}{\textbf{C4 PPL} $\downarrow$} \\
\cmidrule(lr){5-7}\cmidrule(lr){8-10}
 & & & &
\textbf{Llama-2-7B} &
\textbf{Llama-3.1-8B} &
\textbf{Llama-3.2-1B} &
\textbf{Llama-2-7B} &
\textbf{Llama-3.1-8B} &
\textbf{Llama-3.2-1B} \\
\midrule
\multicolumn{3}{l}{BF16 (reference)} & --
  & 5.11 & 5.62 & 8.67
  & 7.04 & 8.82 & 13.53 \\
\midrule

\multirow{28}{*}{\rotatebox[origin=c]{90}{\textbf{4-bit}}}

& \multirow{7}{*}{\shortstack[l]{
    A: \hspace{1.6pt}INT4 \\ W: INT4 \\ $(G{=}64)$}}
  & RTN          & \cmark
    & 5.73 & 7.11 & 12.06
    & 7.97 & 11.99 & 19.69 \\
& & AWQ+         & \xmark
    & 5.57   & 6.83    & 12.46
    & 7.73   & 11.51    & 20.24   \\
& & SmoothQuant  & \xmark
    & 5.68 & 6.98 & 13.37
    & 7.91 & 11.75 & 21.99 \\
& & SINQ+         & \cmark
    & 5.61 & 6.83 & 12.36
    & 7.80 & 11.45 & 20.68   \\
& & QuaRot       & \cmark
    & 6.20 & 7.05 & 12.67
    & 8.89 & 12.10 & 20.96 \\
& & \hl\method{}    & \hl\cmark
    & \hl\textbf{5.56} & \hl\textbf{6.73} & \hl\textbf{11.88}
    & \hl\textbf{7.71} & \hl\textbf{11.35} & \hl\textbf{19.23} \\
\cmidrule(lr){2-10}

& \multirow{7}{*}{\shortstack[l]{
    A: \hspace{1.6pt}MXINT4 \\ W: MXINT4 \\ $(G{=}32)$}}
  & RTN          & \cmark
    & 5.97   & 7.63 & 17.43
    & 8.28   & 12.92 & 29.72   \\
& & AWQ+         & \xmark
    & 5.74   & 7.28    & 17.18
    & 7.97   & 12.28   & 28.17   \\
& & SmoothQuant  & \xmark
    & 5.91 & 7.48 & 15.79
    & 8.31 & 12.60 & 26.86 \\
& & SINQ+         & \cmark
    & 8.58 & 20.87 & 174.71
    & 14.33 & 34.40 & 261.27   \\
& & QuaRot       & \cmark
    & 12.20 & 8.76 & 21.16
    & 17.35 & 15.36 & 35.50 \\
& & \hl\method{}    & \hl\cmark
    & \hl\textbf{5.76} & \hl\textbf{7.15} & \hl\textbf{13.44}
    & \hl\textbf{7.94} & \hl\textbf{11.98} & \hl\textbf{21.86} \\
\cmidrule(lr){2-10}

& \multirow{7}{*}{\shortstack[l]{
    A: \hspace{1.6pt}MXFP4 \\ W: MXFP4 \\ $(G{=}32)$}}
  & RTN          & \cmark
    & 7.29   & 7.62 & 14.35
    & 11.45   & 12.84 & 23.76   \\
& & AWQ+         & \xmark
    & 6.44   & 7.82    & 13.57
    & 9.10   & 12.94    & 22.26   \\
& & SmoothQuant  & \xmark
    & 6.77 & 7.44 & 13.94
    & 10.67 & 12.51 & 23.11 \\
& & SINQ+         & \cmark
    & 32.88 & 52.76 & 135.57
    & 71.64 & 62.52 & 194.51   \\
& & QuaRot       & \cmark
    & 11.04 & 8.27 & 15.91
    & 14.63 & 14.00 & 26.76 \\
& & \hl\method{}    & \hl\cmark
    & \hl\textbf{6.38} & \hl\textbf{7.35} & \hl\textbf{13.13}
    & \hl\textbf{9.70} & \hl\textbf{12.70} & \hl\textbf{21.66} \\
\cmidrule(lr){2-10}

& \multirow{7}{*}{\shortstack[l]{
    A: \hspace{1.6pt}NVFP4 \\ W: NVFP4 \\ $(G{=}16)$}}
  & RTN          & \cmark
    & 5.70   & 9.04 & 14.10
    & 7.92   & 16.17 & 23.56   \\
& & AWQ+         & \xmark
    & 6.51   &  22.79  & 15.90
    & 9.53   & 41.31    & 26.33   \\
& & SmoothQuant  & \xmark
    & 5.53 & 7.32 & 12.40
    & 7.63 & 12.30 & 20.26 \\
& & SINQ+         & \cmark
    & {Div$\dagger$}  & {Div$\dagger$} & {Div$\dagger$}
    & {Div$\dagger$}  & {Div$\dagger$} & {Div$\dagger$} \\
& & QuaRot       & \cmark
    & {Div$\dagger$} & {Div$\dagger$} & {Div$\dagger$}
    & {Div$\dagger$} & {Div$\dagger$} & {Div$\dagger$} \\
& & \hl\method{}    & \hl\cmark
    & \hl\textbf{5.54} & \hl\textbf{6.83} & \hl\textbf{10.91}
    & \hl\textbf{7.64} & \hl\textbf{11.53} & \hl\textbf{17.45} \\
\bottomrule
\end{tabular}%
}
\end{table}

%% file: tables/downstream_main.tex
\begin{table}[t]
\centering
\caption{Llama-3.1-8B downstream accuracy (0-shot, \%, higher is better $\uparrow$) 
for RTN, AWQ+, SmoothQuant, SINQ+, \method{}, and QuaRot, across all four evaluated formats; QuaRot is omitted entirely for NVFP4, since QuaRot's Hadamard rotation redistributes activation magnitude in a way that collapses NVFP4's narrow-range E4M3 block scale. \textbf{Bold}: best calibration-free (CF) result per column within each model+format block. \cmark: calibration-free. \xmark: requires calibration data.}
\label{tab:downstream}
\scriptsize
\setlength{\tabcolsep}{3pt}
\renewcommand{\arraystretch}{0.8}
{%
\begin{tabular}{l l c r r r r r}
\toprule
\multicolumn{8}{c}{\textbf{Llama-3.1-8B}} \\
\midrule
Format & Method & CF & MMLU & HellaSwag & PIQA & ARC-Easy & \green{WinoGrande} \\
\midrule

\multirow{7}{*}{INT4 (G=64)}
  & RTN & \cmark & 54.89 & 56.44 & \bfseries 77.86 & 77.40 & 68.90 \\
& AWQ+ & \xmark & 55.78 & 56.65 & 78.51 & 77.19 & 69.06 \\
& SmoothQuant & \xmark & 55.24 & 55.93 & 77.31 & 77.44 & 68.75 \\
& SINQ+ & \cmark & 55.50 & 56.68 & 77.09 & 77.65 & 67.72 \\
& QuaRot & \cmark & 54.21 & 56.16 & 77.04 & 77.40 & \bfseries 70.17 \\
& \hl\method{} & \hl\cmark & \hl \bfseries 55.61 & \hl \bfseries 56.83 & \hl 77.37 & \hl \bfseries 77.68 & \hl 69.77 \\
\midrule
\multirow{7}{*}{MXINT4 (G=32)}
  & RTN & \cmark & 50.94 & 55.06 & 76.39 & 74.79 & 68.27 \\
& AWQ+ & \xmark & 52.30 & 55.31 & 76.77 & 73.82 & 69.61 \\
& SmoothQuant & \xmark & 52.93 & 54.83 & 76.61 & 75.88 & 66.14 \\
& SINQ+ & \cmark & 30.32 & 42.41 & 68.55 & 63.26 & 58.33 \\
& QuaRot & \cmark & 43.53 & 50.16 & 74.76 & 71.17 & 67.48 \\
& \hl\method{} & \hl\cmark & \hl \bfseries 53.53 & \hl \bfseries 55.28 & \hl \bfseries 77.37 & \hl \bfseries 76.89 & \hl \bfseries 70.32 \\
\midrule
\multirow{7}{*}{MXFP4 (G=32)}
  & RTN & \cmark & 51.51 &  54.17 & 76.71 & 76.43 & 68.03 \\
& AWQ+ & \xmark & 49.85 & 52.75 & 76.39 & 75.34 & 68.11 \\
& SmoothQuant & \xmark & 51.60 & 54.11 & 76.12 & 73.86 & 70.09 \\
& SINQ+ & \cmark & 25.36 & 37.60 & 64.36 & 54.21 & 55.64 \\
& QuaRot & \cmark & 48.69 & 52.81 & 75.84 & \bfseries 77.19 & 68.90 \\
& \hl\method{} & \hl\cmark & \hl \bfseries 52.10 & \hl \bfseries 54.66 & \hl \bfseries 76.77 & \hl 76.22 & \hl \bfseries 69.69 \\
\midrule
\multirow{5}{*}{NVFP4 (G=16)}
  & RTN & \cmark & 43.87 & 53.02 & 74.81 & 71.80 & 65.67 \\
& AWQ+ & \xmark & 31.19 & 45.10 & 71.06 & 63.17 & 62.75 \\
& SmoothQuant & \xmark & 55.67 & 55.60 & 77.53 & 75.80 & 68.90 \\
& SINQ+ & \cmark & 25.51 & 25.41 & 53.43 & 24.62 & 49.17 \\
& \hl\method{} & \hl\cmark & \hl \bfseries 57.76 & \hl \bfseries 56.65 & \hl \bfseries 78.24 & \hl \bfseries 77.90 & \hl \bfseries 71.67 \\

\bottomrule
\end{tabular}}
\end{table}

%% file: appendix/optimization_bcd.tex
\newpage
\section{Optimization background}
\label{app:optimization}

The dual-scale optimization problem introduced in Section~\ref{sec:bcd}
is non-convex and involves a non-differentiable quantization operator $Q(\cdot)$,
making direct gradient-based optimization infeasible. We address this by
combining two classical techniques from optimization: \emph{variable splitting},
which decouples the conflicting roles of the scale vectors, and \emph{block
coordinate descent}, which exploits the resulting structure to obtain
closed-form updates. We briefly review both tools here.

\subsection{Variable splitting}
\label{app:variable_splitting}

Variable splitting~\citep{boyd2011admm, tseng1991applications} is a technique
for handling optimization problems where a single variable plays multiple
conflicting roles in the objective or constraints. The idea is to introduce
separate copies of the variable, one for each role, and enforce equality
between the copies as an explicit constraint. This transforms a tightly coupled,
difficult problem into a more structured one where each copy can be optimized
independently given the others, subject to the coupling constraint.

In our dual-scale formulation in \eqref{eq:primal}, each scale vector plays two distinct roles: as
an \emph{inner} scale applied to the argument of $Q(\cdot)$, which determines the
quantization grid assignment, i.e., which discrete value each weight is mapped to, and
as an \emph{outer} scale applied after applying $Q(\cdot)$, which controls how those
discrete values are reconstructed in full precision. These two roles are in
tension: the inner scale must be chosen to minimize quantization error by
aligning the weight distribution with the quantization grid, while the optimal
outer scale, given a fixed quantized matrix, is determined by a least-squares
criterion that is blind to the quantization grid itself. Because both roles are
played by the same variable, the objective in \eqref{eq:primal}
is non-differentiable and non-convex in $(\boldsymbol{\alpha}, \boldsymbol{\beta})$
jointly.

Variable splitting resolves this conflict by introducing independent copies
$(\boldsymbol{\alpha}_{\mathrm{in}}, \boldsymbol{\alpha}_{\mathrm{out}})$ and
$(\boldsymbol{\beta}_{\mathrm{in}}, \boldsymbol{\beta}_{\mathrm{out}})$,
one pair for each scaling axis (row and column), and enforcing the consistency constraints
$\boldsymbol{\alpha}_{\mathrm{in}} = \boldsymbol{\alpha}_{\mathrm{out}}$ and
$\boldsymbol{\beta}_{\mathrm{in}} = \boldsymbol{\beta}_{\mathrm{out}}$
explicitly. With the inner scales (and hence
$\hat{\mathbf{Q}}_{\boldsymbol{\alpha}_{\mathrm{in}},\boldsymbol{\beta}_{\mathrm{in}}}$
held fixed,)
the relaxed objective becomes bilinear in the outer scales
$\boldsymbol{\alpha}_{\mathrm{out}}^{-1}$ and
$\boldsymbol{\beta}_{\mathrm{out}}^{-1}$, and decomposes into independent
scalar least-squares problems per row and per column, each with a closed-form
solution as derived in Section~\ref{sec:bcd}. After each pair of
closed-form outer scale updates, the consistency constraints are re-enforced
by synchronizing the inner scales with the newly computed outer scales, after
which $\hat{\mathbf{Q}}_{\boldsymbol{\alpha}_{\mathrm{in}},\boldsymbol{\beta}_{\mathrm{in}}}$
is recomputed before the next BCD iteration. For brevity, we drop the subscripts
and write $\hat{\mathbf{Q}}$ hereafter.
This two-step structure is what makes the overall algorithm tractable and efficient in practice. The iterative procedure for cycling through these updates is formalised as Block Coordinate Descent, described next.

\subsection{Block coordinate descent}
\label{app:bcd}

Block Coordinate Descent (BCD)~\citep{tseng2001convergence} is a classical
iterative optimization strategy for problems involving multiple groups of
variables. Rather than jointly optimizing all variables simultaneously, which may be intractable or computationally prohibitive, BCD partitions the variable set into blocks
and cycles through them, optimizing each block exactly (or approximately) while
holding the remaining blocks fixed. Formally, given an objective
$f(\mathbf{u}_1, \mathbf{u}_2, \ldots, \mathbf{u}_K)$, BCD performs the update
\begin{equation}
    \mathbf{u}_k^{(t+1)} =
    \arg\min_{\mathbf{u}_k}\;\;
    f\!\left(\mathbf{u}_1^{(t+1)}, \ldots, \mathbf{u}_{k-1}^{(t+1)},\,
    \mathbf{u}_k,\, \mathbf{u}_{k+1}^{(t)}, \ldots, \mathbf{u}_K^{(t)}\right)
\end{equation}
for each block $k = 1, \ldots, K$ in sequence, using the most recently updated
values of all other blocks. Under mild regularity conditions, in particular when each subproblem admits a unique minimizer, BCD is guaranteed to converge to a stationary point of the objective~\citep{tseng2001convergence}.

The practical appeal of BCD lies in its ability to exploit problem structure:
even when the joint problem is difficult, individual subproblems are often tractable. In our setting, once the inner scales and the quantized matrix
$\hat{\mathbf{Q}}$ are held fixed, the objective in~\eqref{eq:split} decomposes into independent scalar least-squares
subproblems for each row scale $\alpha_{\mathrm{out},i}$ and each column scale
$\beta_{\mathrm{out},j}$, each of which admits a closed-form solution. This
makes BCD a natural and efficient solver for the dual-scale problem.

%% file: appendix/closed_form_solution.tex
\section{Closed-form optimal solutions for scales}
\label{app:closed_form}
This appendix derives the closed-form updates
via variable splitting. We introduce independent inner and outer
copies of each scale vector. With inner scales
$(\boldsymbol{\alpha}_{\mathrm{in}}, \boldsymbol{\beta}_{\mathrm{in}})$
fixed, the objective function is
\begin{equation}
    \mathcal{L}(\boldsymbol{\alpha}_{\mathrm{out}},
    \boldsymbol{\beta}_{\mathrm{out}}) =
    \left\|\mathbf{W} -
    \mathrm{Diag}(\boldsymbol{\alpha}_{\mathrm{out}})^{-1}
    \hat{\mathbf{Q}}\,
    \mathrm{Diag}(\boldsymbol{\beta}_{\mathrm{out}})^{-1}
    \right\|_F^2.
\end{equation}

\paragraph{Closed-Form Update for $\boldsymbol{\alpha}_{\mathrm{out}}$.}
With $\boldsymbol{\beta}_{\mathrm{out}}$ and $\hat{\mathbf{Q}}$ fixed, we
define the auxiliary matrix
$\mathbf{R} = \hat{\mathbf{Q}}\,\mathrm{Diag}(\boldsymbol{\beta}_{\mathrm{out}})^{-1}$,
which absorbs the fixed column scaling into the quantized values. The objective
in~\eqref{eq:split} then reduces to
\begin{equation}
    \mathrm{MSE}_{\mathbf{W}}(\boldsymbol{\alpha}_{\mathrm{out}}) =
    \left\|\mathbf{W} - \mathrm{Diag}(\boldsymbol{\alpha}_{\mathrm{out}})^{-1}
    \mathbf{R} \right\|_F^2
    = \sum_{i=1}^{m}
    \left\|\mathbf{W}_{i,:} - \frac{1}{\alpha_{\mathrm{out},i}}\mathbf{R}_{i,:}
    \right\|_2^2.
\end{equation}
This objective decomposes into $m$ independent scalar
subproblems, one per row, each involving only the scalar $\alpha_{\mathrm{out},i}$.
To cast each subproblem as a standard least-squares problem, we substitute
$u_i = \alpha_{\mathrm{out},i}^{-1}$, which linearizes the dependence on the
scale. Setting the derivative with respect to $u_i$ to zero:
\begin{equation}
    \frac{\partial\, \mathrm{MSE}_{\mathbf{W}}}{\partial u_i} =
    -2\,\mathbf{R}_{i,:}\!\left(\mathbf{W}_{i,:} -
    u_i\,\mathbf{R}_{i,:}\right)^\top = 0,
\end{equation}
and solving for $u_i$ gives
\begin{equation}
    u_i^* = \frac{\mathbf{W}_{i,:}\,\mathbf{R}_{i,:}^\top}
    {\|\mathbf{R}_{i,:}\|_2^2}.
\end{equation}
Inverting back to the original parametrization yields the closed-form update
for the $i$-th row scale:
\begin{equation}
    \alpha_{\mathrm{out},i}^* =
    \frac{\|\mathbf{R}_{i,:}\|_2^2}
    {\mathbf{W}_{i,:}\,\mathbf{R}_{i,:}^\top},
    \quad\text{where}\quad
    \mathbf{R}_{i,:} =
    \hat{\mathbf{Q}}_{i,:}\,\mathrm{Diag}(\boldsymbol{\beta}_{\mathrm{out}})^{-1}.
\end{equation}
Intuitively, $\alpha_{\mathrm{out},i}^*$ scales the $i$-th row so that the quantized row $\mathbf{R}_{i,:}$ best approximates the original row $\mathbf{W}_{i,:}$ in the least-squares sense: it is large when the quantized row underestimates the original, and small when it overestimates.
After computing $\boldsymbol{\alpha}_{\mathrm{out}}^*$, the format constraint
$\boldsymbol{\alpha}_{\mathrm{out}}^{\odot -1} \in
    \mathbb{F}_\alpha^m\;\;$ is enforced by projecting
each $1/ \alpha_{\mathrm{out},i}^*$ to the nearest representable value in
$\mathbb{F}_{\alpha}$.

\paragraph{Closed-Form Update for $\boldsymbol{\beta}_{\mathrm{out}}$.}
The derivation for the column scales follows the same structure. With
$\boldsymbol{\alpha}_{\mathrm{out}}$ and $\hat{\mathbf{Q}}$ fixed, we define
$\mathbf{S} = \mathrm{Diag}(\boldsymbol{\alpha}_{\mathrm{out}})^{-1}\hat{\mathbf{Q}}$,
which absorbs the fixed row scaling into the quantized values. The objective
reduces to
\begin{equation}
    \mathrm{MSE}_{\mathbf{W}}(\boldsymbol{\beta}_{\mathrm{out}}) =
    \left\|\mathbf{W} -
    \mathbf{S}\,\mathrm{Diag}(\boldsymbol{\beta}_{\mathrm{out}})^{-1}
    \right\|_F^2
    = \sum_{j=1}^{n}
    \left\|\mathbf{W}_{:,j} - \frac{1}{\beta_{\mathrm{out},j}}\mathbf{S}_{:,j}
    \right\|_2^2,
\end{equation}
which decomposes into $n$ independent scalar subproblems, one per column.
Substituting $v_j = \beta_{\mathrm{out},j}^{-1}$ and setting the derivative
with respect to $v_j$ to zero:
\begin{equation}
    \frac{\partial\, \mathrm{MSE}_{\mathbf{W}}}{\partial v_j} =
    -2\,\mathbf{S}_{:,j}^\top\!\left(\mathbf{W}_{:,j} -
    v_j\,\mathbf{S}_{:,j}\right) = 0,
\end{equation}
and solving for $v_j$ gives
\begin{equation}
    v_j^* = \frac{\mathbf{S}_{:,j}^\top\,\mathbf{W}_{:,j}}
    {\|\mathbf{S}_{:,j}\|_2^2}.
\end{equation}
Inverting back yields the closed-form update for the $j$-th column scale:
\begin{equation}
    \beta_{\mathrm{out},j}^* =
    \frac{\|\mathbf{S}_{:,j}\|_2^2}
    {\mathbf{S}_{:,j}^\top\,\mathbf{W}_{:,j}},
    \quad\text{where}\quad
    \mathbf{S}_{:,j} =
    \mathrm{Diag}(\boldsymbol{\alpha}_{\mathrm{out}})^{-1}\hat{\mathbf{Q}}_{:,j}.
\end{equation}
As with the row scales, the format constraint $\boldsymbol{\beta}_{\mathrm{out}}^{\odot -1} \in
    \mathbb{F}_\beta^n\;\;$ is enforced by projecting each $1/\beta_{\mathrm{out},j}^*$
to the nearest representable value in $\mathbb{F}_{\beta}$ after the update.

%% file: appendix/theory1_proof.tex
\section{Proof of Theorem~\ref{thm:convergence}.}
\label{app:proof_theory1}

We first establish the uniqueness of the per-block subproblem
solutions (Lemmas~\ref{lem:alpha_unique} and~\ref{lem:beta_unique}),
then prove the theorem. For brevity, define
\begin{equation}
    g(\boldsymbol{\alpha}, \boldsymbol{\beta}) :=
    \bigl\|\mathbf{W} -
    \mathrm{Diag}(\boldsymbol{\alpha})^{-1}\,
    \hat{\mathbf{Q}}^{(T)}\,
    \mathrm{Diag}(\boldsymbol{\beta})^{-1}\bigr\|_F^2,
\end{equation}
the objective \eqref{eq:primal} with $\hat{\mathbf{Q}}$
\emph{frozen} at its post-stabilization value
$\hat{\mathbf{Q}}^{(T)}$. By the stabilization hypothesis, as empirically verified in Figure~\ref{fig:qhat_stable} (main text) and Figure~\ref{fig:llama_qhat_stable} (appendix),
$\hat{\mathbf{Q}}^{(t)} = \hat{\mathbf{Q}}^{(T)}$ for all
$t \geq T$, so for $t \geq T$ we have
$g(\boldsymbol{\alpha}^{(t)}, \boldsymbol{\beta}^{(t)}) =
\mathrm{MSE}_{\mathbf{W}}(\boldsymbol{\alpha}^{(t)},
\boldsymbol{\beta}^{(t)})$. For $t < T$, $g$ and
$\mathrm{MSE}_{\mathbf{W}}$ differ in general because
$\hat{\mathbf{Q}}^{(t)} \neq \hat{\mathbf{Q}}^{(T)}$. Working
with the function $g$ in the post-stabilization regime
allows us to invoke 
BCD convergence theorem, which
requires a stationary objective.
\begin{figure}[t]
\centering
\includegraphics[width=0.97\linewidth]{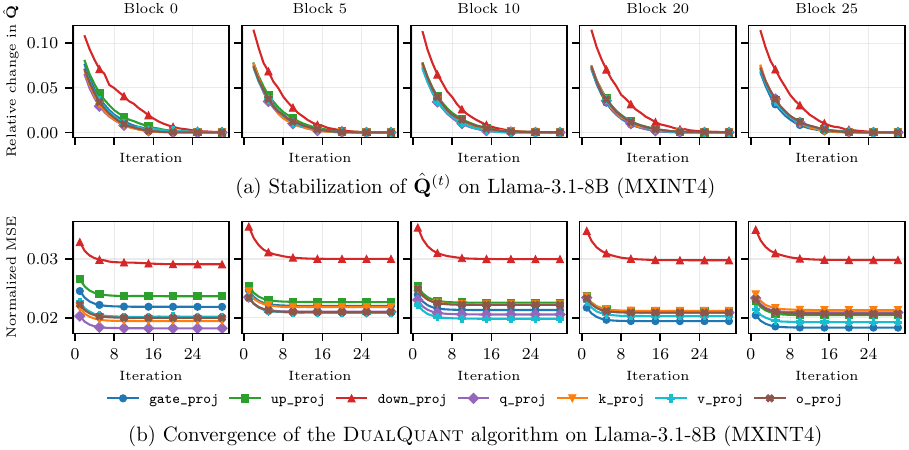}
\caption{\textbf{(a)} Relative change $ \scriptstyle \|\hat{\mathbf{Q}}^{(t)} -
\hat{\mathbf{Q}}^{(t-1)}\|_F / \|\hat{\mathbf{Q}}^{(t-1)}\|_F$
vs.\ BCD iteration for five transformer blocks.
The $\hat{\mathbf{Q}}^{(t)}$ change decays to near zero by $t{\approx}16$ for all layers and blocks, validating the stabilization condition of Theorem~\ref{thm:convergence}.
\textbf{(b)} Normalized MSE
$ \scriptstyle \|\mathbf{W} - \hat{\mathbf{W}}\|_F^2 / \|\mathbf{W}\|_F^2$
vs.\ BCD iteration for the same blocks. All curves decrease monotonically and plateau within ${\approx}10$ iterations,
confirming Theorem~\ref{thm:convergence}(i).
Together with the Qwen2.5-7B results in Figure~\ref{fig:qhat_stable}, this illustrates that the convergence behavior of \method{} is consistent across the two model families evaluated here. Notably, \texttt{down\_proj} consistently shows the highest residual MSE,
consistent with it being among the harder projections to
quantize~\citep{frantar2022gptq}.
}
\label{fig:llama_qhat_stable}
\end{figure}

\begin{lemma}[Unique row-scale minimizer]
\label{lem:alpha_unique}
Under Assumption~\ref{ass:nondegenerate}, for any fixed
$\boldsymbol{\beta} \in \mathbb{R}_{+}^{n}$, the subproblem
$\min_{\boldsymbol{\alpha}} g(\boldsymbol{\alpha},
\boldsymbol{\beta})$ admits a unique minimizer
$\boldsymbol{\alpha}^* \in \mathbb{R}_{+}^{m}$.
\end{lemma}

\begin{proof}
The objective decomposes across rows as
$g(\boldsymbol{\alpha}, \boldsymbol{\beta}) =
\sum_{i=1}^{m} \|\alpha_i^{-1} \mathbf{R}_{i,:} -
\mathbf{W}_{i,:}\|_2^2$, where $\mathbf{R}_{i,:} =
\hat{\mathbf{Q}}^{(T)}_{i,:} \mathrm{Diag}(\boldsymbol{\beta})^{-1}$.
Substituting $u_i = \alpha_i^{-1}$ yields a strictly convex
quadratic in $u_i$ with unique minimizer
$u_i^* = \mathbf{W}_{i,:}\,\mathbf{R}_{i,:}^\top /
\|\mathbf{R}_{i,:}\|_2^2$, which is positive under
Assumption~\ref{ass:nondegenerate}. Specifically, Assumption~\ref{ass:nondegenerate} ensures $\|\mathbf{R}_{i,:}\|_2^2 > 0$ (non-zero denominator). The numerator $\mathbf{W}_{i,:}\,\mathbf{R}_{i,:}^\top$ is positive because all quantization formats considered use sign-preserving quantizers, meaning $Q(x)$ and $x$ share the same sign for any nonzero $x$. Since $\alpha_i, \beta_j > 0$, it follows that $W_{i,k}$ and $R_{i,k}$ share the same sign for all $k$, making the inner product strictly positive. Therefore $\alpha_i^* =
1/u_i^* \in \mathbb{R}_{+}$ is the unique minimizer.
\end{proof}

\begin{lemma}[Unique column-scale minimizer]
\label{lem:beta_unique}
Under Assumption~\ref{ass:nondegenerate}, for any fixed
$\boldsymbol{\alpha} \in \mathbb{R}_{+}^{m}$, the subproblem
$\min_{\boldsymbol{\beta}} g(\boldsymbol{\alpha},
\boldsymbol{\beta})$ admits a unique minimizer
$\boldsymbol{\beta}^* \in \mathbb{R}_{+}^{n}$.
\end{lemma}

\begin{proof}
Symmetric to the proof of Lemma~\ref{lem:alpha_unique}, with
rows and columns interchanged.
\end{proof}

\begin{proof}[Proof of Theorem~\ref{thm:convergence}]
We prove the two parts of the theorem in turn: monotone
descent~(i) and stationarity of limit points~(ii).

\medskip
\noindent\textbf{Part (i) --- monotone descent.}
At iteration $t \geq T$ with $\boldsymbol{\beta}^{(t)}$ and
$\hat{\mathbf{Q}}^{(T)}$ fixed,
Algorithm~\ref{alg:dual_scale} sets $\boldsymbol{\alpha}^{(t+1)}$
to the unique minimizer of $g(\boldsymbol{\alpha},
\boldsymbol{\beta}^{(t)})$ over $\boldsymbol{\alpha}$
(Lemma~\ref{lem:alpha_unique}). Since
$\boldsymbol{\alpha}^{(t)}$ is feasible:
\begin{equation}
    g(\boldsymbol{\alpha}^{(t+1)}, \boldsymbol{\beta}^{(t)})
    \leq
    g(\boldsymbol{\alpha}^{(t)}, \boldsymbol{\beta}^{(t)}).
    \label{eq:descent_alpha}
\end{equation}
By the symmetric argument, with $\boldsymbol{\alpha}^{(t+1)}$
and $\hat{\mathbf{Q}}^{(T)}$ fixed,
Algorithm~\ref{alg:dual_scale} sets $\boldsymbol{\beta}^{(t+1)}$
to the unique minimizer of $g(\boldsymbol{\alpha}^{(t+1)},
\boldsymbol{\beta})$ over $\boldsymbol{\beta}$
(Lemma~\ref{lem:beta_unique}), giving:
\begin{equation}
    g(\boldsymbol{\alpha}^{(t+1)}, \boldsymbol{\beta}^{(t+1)})
    \leq
    g(\boldsymbol{\alpha}^{(t+1)}, \boldsymbol{\beta}^{(t)}).
    \label{eq:descent_beta}
\end{equation}
Chaining \eqref{eq:descent_alpha} and \eqref{eq:descent_beta}:
\begin{equation}
    g(\boldsymbol{\alpha}^{(t+1)}, \boldsymbol{\beta}^{(t+1)})
    \leq
    g(\boldsymbol{\alpha}^{(t)}, \boldsymbol{\beta}^{(t)})
    \quad \forall\, t \geq T,
\end{equation}
establishing part~(i). Since $g \geq 0$, the sequence
$\{g(\boldsymbol{\alpha}^{(t)}, \boldsymbol{\beta}^{(t)})\}_{t
\geq T}$ converges to a finite limit $g^* \geq 0$.

\medskip
\noindent\textbf{Part (ii) --- Stationarity of limit points.}
We invoke Theorem~4.1(c) of
~\citep{tseng2001convergence}, which states that for
cyclic block coordinate descent on a continuous objective
where each block subproblem has a unique minimizer, every
limit point of the iterates is a coordinate-wise stationary
point. The three required conditions hold as follows.

\begin{enumerate}[label=(\alph*), nosep, wide=0pt, leftmargin=*]
    \item \textbf{Continuity of $g$.} With $\hat{\mathbf{Q}}^{(T)}$
    fixed, $g(\boldsymbol{\alpha}, \boldsymbol{\beta})$ is a sum
    of rational functions of $\alpha_i$ and $\beta_j$ with
    nonzero denominators on the open positive orthant
    $\{\alpha_i, \beta_j > 0\}$, hence continuously
    differentiable on its domain. All iterates remain in this
    domain by construction under
    Assumption~\ref{ass:nondegenerate}.

    \item \textbf{Unique per-block minimizer.}
    Lemmas~\ref{lem:alpha_unique} and~\ref{lem:beta_unique}
    give uniqueness for both block subproblems. Uniqueness follows from the strict convexity of each scalar subproblem after the substitution $u_i = \alpha_i^{-1}$ (resp.\ $v_j = \beta_j^{-1}$), which yields a strictly convex quadratic with a single minimizer.

    \item \textbf{Boundedness of the iterates.}
    The level set $\{(\boldsymbol{\alpha}, \boldsymbol{\beta}) :
    g(\boldsymbol{\alpha}, \boldsymbol{\beta}) \leq
    g(\boldsymbol{\alpha}^{(T)}, \boldsymbol{\beta}^{(T)})\}$ is
    compact \footnote{\mahsa{After fixing the trivial reciprocal scaling ambiguity $((\boldsymbol{\alpha},\boldsymbol{\beta})\sim(c\boldsymbol{\alpha},\boldsymbol{\beta}/c)$, e.g., by normalizing one of the scale vectors.}}: $g \to \infty$ as any $\alpha_i$ or $\beta_j$
    approaches $0$ or $\infty$ (whenever $\hat{\mathbf{Q}}^{(T)}
    \neq \mathbf{0}$, ensured by
    Assumption~\ref{ass:nondegenerate}). Combined with the
    monotone decrease of part~(i), all iterates remain in this
    compact set.
\end{enumerate}

By Theorem~4.1(c) of
~\citep{tseng2001convergence}, every limit point $(\boldsymbol{\alpha}^*,
\boldsymbol{\beta}^*)$ of $\{(\boldsymbol{\alpha}^{(t)},
\boldsymbol{\beta}^{(t)})\}_{t \geq T}$ satisfies the
coordinate-wise stationarity conditions
$\boldsymbol{\alpha}^* = \arg\min_{\boldsymbol{\alpha}}
g(\boldsymbol{\alpha}, \boldsymbol{\beta}^*)$ and
$\boldsymbol{\beta}^* = \arg\min_{\boldsymbol{\beta}}
g(\boldsymbol{\alpha}^*, \boldsymbol{\beta})$, completing the
proof of part~(ii).
\end{proof}

\begin{remark}[Format-constrained projection]
\label{rem:projection}
Theorem~\ref{thm:convergence} is stated for unconstrained $\boldsymbol{\alpha} \in \mathbb{R}_+^m$ and $\boldsymbol{\beta} \in \mathbb{R}_+^n$. In practice, each scale is projected onto the discrete admissible set after each closed-form update. Since each subproblem in $u_i = \alpha_i^{-1}$ (resp.\ $v_j = \beta_j^{-1}$) is a strictly convex quadratic, 
the objective increases monotonically with distance from the unconstrained minimizer. The nearest admissible point to this unconstrained minimizer therefore always achieves the lowest constrained objective, so projecting to the nearest representable value guarantees monotone decrease at each step.
\end{remark}

\begin{remark}[On the stabilization assumption]
\label{rem:stabilization}
Theorem~\ref{thm:convergence} assumes that
$\hat{\mathbf{Q}}^{(t)}$ stabilizes at a finite iteration $T$.
Since $\hat{\mathbf{Q}}$ takes values in the finite set
$\mathbb{F}^{m \times n}$, stabilization is guaranteed
whenever the objective decreases monotonically across
$\hat{\mathbf{Q}}$ updates. We numerically observe this across model families and quantization formats. For example, for Qwen2.5-7B and Llama3.1-8B with MXINT4, Figures~\ref{fig:qhat_stable} and \ref{fig:llama_qhat_stable} show that $\hat{\mathbf{Q}}^{(t)}$ stabilizes within
$\approx 12$--$16$ iterations across every projection layer of every
evaluated decoder block, well within our default of $T = 15$ iterations used in all experiments.
\end{remark}

%% file: appendix/sinq.tex
\section{SINQ baseline: background and extension}
\label{app:sinq}

This appendix documents the SINQ~\citep{muller2025sinq} baseline
used in our experiments. We briefly review SINQ's scale-selection
procedure (Appendix~\ref{app:sinq_method}), describe our extension
to MX and NV formats (Appendix~\ref{app:sinq_extension}), and
explain why this extension exhibits the divergence behavior
observed in Table~\ref{tab:ppl_wa_llama} on those formats
(Appendix~\ref{app:sinq_limitations}).

\subsection{Background}
\label{app:sinq_method}

SINQ is a calibration-free dual-scaling method that selects row
scales $\boldsymbol{\alpha} \in \mathbb{R}_{+}^{m}$ and column
scales $\boldsymbol{\beta} \in \mathbb{R}_{+}^{n}$ to normalize
the weight matrix prior to quantization,
\begin{equation}
\label{eq:sinq_scaled}
    \breve{\mathbf{W}} = 
    \mathrm{Diag}(\boldsymbol{\alpha})^{-1}\,\mathbf{W}\,
    \mathrm{Diag}(\boldsymbol{\beta})^{-1}.
\end{equation}
The scales are obtained by minimizing a \emph{matrix imbalance}
objective that drives the per-row and per-column standard
deviations of $\breve{\mathbf{W}}$ toward uniformity. The
imbalance is reported as the ratio of the largest to the smallest
per-axis standard deviation of $\breve{\mathbf{W}}$,
\begin{equation}
\label{eq:sinq_imbalance}
    \mathcal{I}(\breve{\mathbf{W}}) =
    \frac{\max\bigl(
        \max_i \sigma_i(\breve{\mathbf{W}}),\;
        \max_j \tau_j(\breve{\mathbf{W}})
    \bigr)}
    {\min\bigl(
        \min_i \sigma_i(\breve{\mathbf{W}}),\;
        \min_j \tau_j(\breve{\mathbf{W}})
    \bigr)},
\end{equation}
where $\sigma_i$ and $\tau_j$ denote the standard deviations of
row $i$ and column $j$ of $\breve{\mathbf{W}}$. A perfectly
balanced matrix has $\mathcal{I}(\breve{\mathbf{W}}) = 1$.
SINQ minimizes $\mathcal{I}$ via a Sinkhorn--Knopp-style
alternating normalization~\citep{sinkhorn1967} that alternately
divides each row and column of $\breve{\mathbf{W}}$ by its
current standard deviation, accumulating the reciprocal products
into $\boldsymbol{\alpha}$ and $\boldsymbol{\beta}$. The
procedure is calibration-free, processes each layer
independently, and runs at roughly $1.1\times$ the cost of RTN.

SINQ was originally designed and evaluated for INT3, INT4,
and NF4, all of which store the per-block scales in unconstrained
BF16. 

\subsection{Extending SINQ to W+A and MX/NV formats}
\label{app:sinq_extension}
To produce baseline W+A results for MXINT4, MXFP4,
and NVFP4, we extend SINQ in two directions: \emph{(i)} to
joint W+A quantization by applying SINQ's column scales to
rescale activations before quantization (same format as
weights), and \emph{(ii)} to MX and NV formats by running
SINQ's algorithm to completion and then projecting the resulting
scales to its target data format (E8M0 for MX, FP8 UE4M3 for NVFP4) as a post-processing step. Since SINQ's optimization is unaware of scale format constraints, this projection only partially mitigates the mismatch: on MXINT4 and MXFP4 it leads to noticeable quality degradation, and on NVFP4 the scales
drift entirely outside the FP8 representable range, causing
complete divergence (marked $\dagger$ in the tables). 

Both extensions are our own and have not been validated by the
original authors; SINQ+ results on MX and NV formats should
accordingly be treated as best-effort baselines. We note that
even in its original weight-only setting without any of our
extensions, SINQ produces inferior scales to \method{} across
all formats (Appendix~\ref{app:weight_only}).

\subsection{Why Sinkhorn iteration diverges under format constraints}
\label{app:sinq_limitations}
The large perplexity values observed for SINQ on MXINT4 and
MXFP4 in Tables~\ref{tab:ppl_wa_llama}
and~\ref{tab:ppl_wa_qwen}, and the complete divergence on
NVFP4 (marked $\dagger$), indicate a fundamental
incompatibility between SINQ's format-unaware Sinkhorn
iteration and the discrete scale constraints of MX and NV
formats. In the unconstrained BF16 setting, Sinkhorn--Knopp iteration converges to a unique fixed point $(\boldsymbol{\alpha}^*,
\boldsymbol{\beta}^*)$~\citep{sinkhorn1964relationship}. The
fixed-point scale values lie on a continuum and can be
represented exactly in BF16. Once the admissible set is
restricted to a discrete format such as E8M0, however, this
property is lost. The fixed-point Sinkhorn target generally lies
between two representable E8M0 values, and the projection step
that follows each Sinkhorn iteration snaps the iterate to the
nearest representable scale rather than the true fixed point.
The next Sinkhorn step then re-derives a target that may point away
from the projected iterate, causing the iteration to oscillate between nearby E8M0 representatives without converging. \method{} avoids
this failure mode by construction. Its closed-form coordinate
updates are derived under the format constraint from the outset
and the projection step preserves the monotone descent property
of the BCD scheme (Theorem~\ref{thm:convergence}), so iterates
remain in the admissible set throughout.

\subsection{Comparison of objectives}
\label{app:sinq_vs_dualquant}

Beyond the format-compatibility issue analyzed in
Appendix~\ref{app:sinq_limitations}, SINQ's matrix imbalance
objective \eqref{eq:sinq_imbalance} measures a different quantity from the post-quantization reconstruction error: a lower imbalance does not necessarily translate to a lower MSE, since the imbalance metric is a function of $\breve{\mathbf{W}}$
alone and makes no reference to the quantization grid, which means
the same imbalance value corresponds to different post-quantization
errors under INT4, MXINT4, MXFP4, and NVFP4. Our weight
reconstruction MSE, in contrast, contains the format-specific
quantizer $Q(\cdot)$ explicitly, so the optimization is aware of
the grid that the scales will be evaluated against and can adapt
the scale assignment to non-uniform grids such as those of FP4 and
NVFP4. This explains why a method optimized for low matrix
imbalance does not automatically achieve low quantization error,
as illustrated in Figure~\ref{fig:motivation}. Furthermore, the reconstruction MSE is more directly tied to downstream metrics such as perplexity and accuracy than proxy objectives such as matrix imbalance, which helps explain why \method{}, by minimizing reconstruction error directly, exhibits stronger downstream performance than SINQ across most evaluated scenarios.

%% file: appendix/causal_effect_of_beta.tex
\section{Isolated activation-side effect of $\beta$}
\label{app:beta_causal}

Section~\ref{sec:dual_scale} establishes algebraically that the column scale $\beta$ reshapes
the per-channel statistics of the activations presented to the quantizer, and argues mechanistically that weight-MSE-optimal scales
attenuate the same channels where activation outliers concentrate. To confirm this is a genuine improvement in activation quantization quality, and not merely a by-product of better
weight quantization, we isolate $\beta$'s effect on activations by holding the quantized
weights fixed.

\paragraph{Setup.} We take \method{}'s quantized weight matrix $\hat{Q}$ for each projection
in block 0 of Llama-3.1-8B (INT4, $G{=}64$) as fixed and measure the Signal-to-Quantization-Noise
Ratio (SQNR, dB, higher is better) of the quantized activations with and without applying
$\beta$, using 256 representative input sequences. Block 0 is a representative early block;
the pattern holds across depth. All other conditions are
identical between the two runs.

\paragraph{Results.} Table~\ref{tab:beta_sqnr} reports per-projection activation SQNR.
Applying $\beta$ improves activation SQNR by $+0.33$ dB on average across projections and
by $+0.53$ dB at the block output level. The improvement is consistent across all seven
projection types, with the largest gains on \texttt{k\_proj} ($+0.81$ dB) and
\texttt{v\_proj} ($+0.65$ dB) --- precisely the projections where activation outliers are
most pronounced~\citep{dettmers2022llmint8,sun2024massive}. This confirms that $\beta$
provides a real, partial activation-side benefit that is separate from and additive to its
role in reducing weight quantization error. The gain is not calibration substitution: unlike
SmoothQuant, $\beta$ is derived entirely from the static weight matrix, yet still measurably \green{reduces the activation quantization error}.

\begin{table}[h]
\centering
\caption{Per-projection activation SQNR (dB, $\uparrow$) with and without \method{}'s
$\beta$, on Llama-3.1-8B block~0 (INT4, $G{=}64$). Quantized weights $\hat{Q}$ are held
fixed; only the presence of the $\beta$ pre-scaling varies. $\Delta > 0$ indicates
$\beta$ reduces activation quantization error. \green{Each $\Delta$ is computed at full precision and may differ slightly from the difference of the rounded columns.} Other blocks exhibit the same qualitative behavior.}
\label{tab:beta_sqnr}
\small
\setlength{\tabcolsep}{6pt}
\renewcommand{\arraystretch}{1.1}
\begin{tabular}{@{}lrrr@{}}
\toprule
\textbf{Projection} &
\textbf{$\beta$ applied} &
\textbf{$\beta$ not applied} &
\textbf{$\Delta$ (dB)} \\
\midrule
\texttt{k\_proj}    & 17.61 & 16.80 & $+0.81$ \\
\texttt{v\_proj}    & 17.45 & 16.80 & $+0.65$ \\
\texttt{up\_proj}   & 18.48 & 18.20 & $+0.27$ \\
\texttt{q\_proj}    & 17.05 & 16.80 & $+0.25$ \\
\texttt{gate\_proj} & 18.36 & 18.20 & $+0.15$ \\
\texttt{o\_proj}    & 17.40 & 17.34 & $+0.07$ \\
\texttt{down\_proj} & 16.76 & 16.70 & $+0.06$ \\
\midrule
\textit{Mean}       & 17.59 & 17.26 & $+0.33$ \\
\midrule
\textit{Block Output SQNR} & 17.98 & 17.44 & $+0.53$ \\
\bottomrule
\end{tabular}
\end{table}

%% file: appendix/baseline_methods.tex
\section{Additional details on baselines}
\label{app:baselines}
This appendix details each baseline used in the joint W+A quantization comparison (Tables~\ref{tab:ppl_wa_llama} and~\ref{tab:ppl_wa_qwen}). All calibration-based baselines (AWQ+ and SmoothQuant) use the same calibration set: $256$ sequences of length $8192$ randomly sampled from the WikiText-2 training split, strictly disjoint from the evaluation splits. All baselines use the same block sizes and scale formats as \method{} for each evaluated format (INT4 $G{=}64$; MXINT4/MXFP4 $G{=}32$ with E8M0; NVFP4 $G{=}16$ with FP8 UE4M3).  Below we describe the six baselines in detail.
\begin{itemize}[nosep, wide=0pt, leftmargin=*]
\item \textbf{RTN.} Applies round-to-nearest quantization independently to weights and activations using format-native per-row or per-block absmax scaling, with no rescaling, calibration, or optimization. Activations are quantized dynamically at inference using per-token statistics. RTN serves as the lower-bound baseline.
\item \textbf{AWQ+.} AWQ~\citep{lin2024awq} is a weight-only, activation-aware method that rescales weights based on high-activation channels. We extend it to W+A by adding RTN activation quantization after AWQ rescaling, and denote this method as AWQ+. Scales are computed using the official AWQ search procedure with a 20-candidate grid and clipping ratio search enabled. 

\item \textbf{SmoothQuant.} SmoothQuant~\citep{xiao2023smoothquant} uses calibration data to move activation outliers into weights via per-channel rescaling 
$s_j \propto \max_i |\mathbf{X}_{ij}|^\gamma / \max_i |\mathbf{W}_{ij}|^{1-\gamma}$. We evaluate the migration strength $\gamma \in \{0.5, 0.85\}$ 
as suggested in the original paper~\citep{xiao2023smoothquant}; since $\gamma = 0.5$ consistently yields  better performance across our experiments, we adopt it for SmoothQuant baselines.
After rescaling, weights and activations are quantized using the same block-wise quantizer as \method{}. It serves as the main calibration-based W+A baseline.

\item \textbf{SINQ+.} SINQ~\citep{muller2025sinq} is a
calibration-free weight-only dual-scale method 
originally designed for INT3/INT4 and NF4 with unconstrained BF16 scales. We extend it to W+A and MX/NV formats as described in
Appendix~\ref{app:sinq_extension}; however, since SINQ's
optimization is entirely unaware of scale format constraints,
projecting its output scales to the target data format only partially
mitigates the mismatch, causing degraded results on MXINT4
and MXFP4 and complete divergence on NVFP4 (marked $\dagger$).
SINQ+ results on MX and NV formats should be treated as
best-effort baselines.

\item \textbf{QuaRot.} QuaRot~\citep{ashkboos2024quarot} is a calibration-free W+A method that suppresses activation outliers by applying randomized Hadamard rotations to the hidden states, spreading outlier energy across channels while leaving the network output unchanged, since each rotation is absorbed into the following weights. The residual-stream rotation is fused offline; the value-path and down-projection rotations run online. Following rotation, weights and activations are quantized using the same block-wise quantizer as \method{}, with activations quantized dynamically based on post-rotation per-token statistics. QuaRot serves as the rotation-based counterpart to RTN, isolating the effect of rotation without weight-error compensation.


\end{itemize}

%% file: appendix/additional_reults.tex
\section{Additional W+A quantization results}
\label{app:additional_results}
This appendix provides supplementary quantization results
complementing the main experiments of
Section~\ref{sec:wa_quant}. We report 4-bit W+A results on
the Qwen model family and 8-bit W+A results across both
families, confirming that \method{}'s calibration-free
activation rescaling generalizes across model architectures,
quantization formats, and bit-widths.

\subsection{4-bit W+A quantization}
\label{app:4bit_wa}
\mahsa{
\paragraph{Perplexity.}
Table~\ref{tab:ppl_wa_qwen} reports WikiText-2 and C4
perplexity for joint W+A quantization on the Qwen family
across all four 4-bit formats. Results are broadly
consistent with the Llama family (Table~\ref{tab:ppl_wa_llama} in the main text): \method{} closely matches or surpasses 
SmoothQuant on the MX and NV formats without calibration
data, and outperforms the calibration-free baselines QuaRot and SINQ+ in most
settings, as these baselines frequently diverge or degrade severely on the block formats.
On NVFP4,
\method{} surpasses SmoothQuant on Qwen2.5-7B and Qwen3-0.6B (e.g., $6.53$ vs.\ $6.70$
and $23.77$ vs.\ $24.16$ on WikiText-2) and closely matches it on Qwen3-14B
($7.75$ vs.\ $7.73$), while SINQ+ diverges on all three models and QuaRot either
diverges or degrades drastically (e.g., $12.11$ and $173.38$ on Qwen2.5-7B and
Qwen3-0.6B), leaving \method{} as the strongest calibration-free method on this
format.
On MXFP4, \method{} outperforms SmoothQuant on Qwen3-0.6B (e.g., $27.46$ vs.\ $28.66$ on WikiText-2)
and achieves substantially better C4 perplexity ($42.59$ vs.\ $47.83$), while
SINQ+ collapses ($195.98$) and QuaRot degrades substantially ($68.96$); on
Qwen2.5-7B, QuaRot remains competitive ($7.38$), narrowly ahead of \method{}
($7.45$). On MXINT4, \method{} outperforms SmoothQuant on Qwen3-0.6B ($32.99$
vs.\ $33.63$) and matches it on Qwen3-14B ($8.42$ vs.\ $8.42$), while SINQ+
degrades severely ($353.83$) and QuaRot collapses
($131.31$). On INT4, \method{} remains competitive with SmoothQuant across all
three models, though slightly below on Qwen2.5-7B ($7.38$ vs.\ $7.10$); QuaRot
achieves the best calibration-free result on Qwen2.5-7B ($6.97$) but degrades
substantially on Qwen3-0.6B ($48.47$), while SINQ+ remains competitive on this
format. The results on C4 are consistent with WikiText-2, confirming that the
calibration-free activation rescaling generalizes across evaluation distributions. Overall, these results show that \method{} provides 
consistently strong calibration-free
W+A quantization across formats without relying on calibration data or format-specific rotation behavior.
}

\input{tables/main_table_Qwen_family}

\mahsa{
\paragraph{Zero-shot downstream tasks.}
Table~\ref{tab:downstream_qwen} reports zero-shot accuracy on MMLU, HellaSwag, PIQA,
ARC-Easy, and \green{WinoGrande} for Qwen3-14B across all four W+A quantization formats. Results
are consistent with the Llama-3.1-8B results (Table~\ref{tab:downstream} in the main text):
\method{} achieves the best calibration-free accuracy on nearly all format-task combinations.
On MXFP4 and MXINT4, \method{} outperforms both calibration-free baselines
QuaRot and SINQ+ on every task, often by a large margin, and stays close to calibration-based
SmoothQuant, surpassing it on \green{WinoGrande} (72.06 vs.\ 69.14 on MXFP4 and 70.64 vs.\ 66.30 on
MXINT4) while trailing narrowly on the remaining tasks (e.g., 72.91 vs.\ 73.26 MMLU on MXFP4
and 71.47 vs.\ 72.50 MMLU on MXINT4). On NVFP4,
\method{} is the strongest calibration-free method: SINQ+ collapses
to near-random accuracy (23.39 MMLU), whereas \method{} stays competitive with calibration-based
SmoothQuant (e.g., 75.06 vs.\ 75.97 MMLU, 78.99 vs.\ 78.94 PIQA), closing most of the gap
without any access to calibration data. On INT4, \method{} matches SmoothQuant closely
across all five tasks and attains the best calibration-free accuracy on four of the five, trailing
only SINQ+ on PIQA (76.71 vs.\ 77.80).
As on Llama, SINQ+ degrades severely on MX formats (62.38--62.58 MMLU on MXINT4 and
MXFP4) and collapses on NVFP4, confirming that the perplexity improvements of
\method{} over calibration-free baselines translate directly to downstream task quality.}
\input{tables/downstream_qwen_appndx}

\subsection{8-bit W+A quantization}
\label{app:8bit_wa}

\input{tables/ppl_8bit_weight_activation}
\green{Table~\ref{tab:ppl_8bit_wa} reports WikiText-2 perplexity for joint W+A
quantization at 8-bit precision across INT8, MXINT8, and MXFP8. At 8 bits the
quantization error is already small, so there is limited room for improvement
over BF16; nonetheless, \method{} remains the strongest calibration-free method
and is effectively lossless. On MXINT8 and MXFP8 it matches or surpasses the
calibration-based SmoothQuant on every model except MXFP8 Qwen3-14B, where it
trails by only $0.01$ ($7.64$ vs.\ $7.63$), all without any calibration data;
representative gains include $18.13$ vs.\ $18.18$ on Qwen3-0.6B and $7.53$ vs.\
$7.55$ on Qwen3-14B for MXINT8. On MXINT8, \method{} tracks the BF16 reference
almost exactly (e.g., $6.15$ vs.\ $6.15$ on Qwen2.5-7B and $18.13$ vs.\ $18.14$
on Qwen3-0.6B), confirming near-lossless quantization. In contrast, the other
calibration-free baseline, SINQ+, degrades severely on the MX formats even at
8 bits (e.g., $47.87$ and $53.41$ on Qwen3-0.6B for MXINT8 and MXFP8),
confirming that its format-unaware scales remain incompatible with MX block
constraints regardless of bit-width and underscoring the benefit of \method{}'s
format-aware dual scaling.}

%% file: tables/main_table_Qwen_family.tex
\begin{table}[t]
\centering
\caption{WikiText-2 and C4 perplexity ($\downarrow$) for joint weight-and-activation quantization at 4-bit precision on the Qwen family. \method{} applies $\boldsymbol{\beta}$-rescaling to activations without calibration data. SmoothQuant and AWQ+ use activation statistics from 256 calibration samples. SINQ+ is our extension of SINQ to W+A. QuaRot rotates activations/weights online with a Hadamard transform (no calibration).
\textbf{Bold}: best calibration-free (\textbf{CF}) result. ($\dagger$): degenerate result $(PPL>1000)$; method incompatible with format. \cmark: calibration-free. \xmark: requires calibration data.}
\label{tab:ppl_wa_qwen}
\setlength{\tabcolsep}{4pt}
\renewcommand{\arraystretch}{1.1}
\resizebox{0.92\textwidth}{!}{%
\begin{tabular}{@{}cllc cccccc@{}}
\toprule
& & & & \multicolumn{3}{c}{\textbf{WikiText-2 PPL} $\downarrow$} & \multicolumn{3}{c}{\textbf{C4 PPL} $\downarrow$} \\
\cmidrule(lr){5-7}\cmidrule(lr){8-10}
\textbf{Bits} & \textbf{Format} & \textbf{Method} & \textbf{CF} & \textbf{Qwen2.5-7B} & \textbf{Qwen3-0.6B} & \textbf{Qwen3-14B} & \textbf{Qwen2.5-7B} & \textbf{Qwen3-0.6B} & \textbf{Qwen3-14B} \\
\midrule
\multicolumn{3}{l}{BF16 (reference)} & --
  & 6.15 & 18.14 & 7.54
  & 9.12 & 27.30 & 11.06 \\
\midrule

\multirow{28}{*}{\rotatebox[origin=c]{90}{\textbf{4-bit}}}

& \multirow{7}{*}{\shortstack[l]{A: INT4 \\ W: INT4 \\ $(G{=}64)$}}
  & RTN & \cmark
    & 7.47 & 49.44 & 8.39
    & 13.50 & 74.78 & 14.83 \\
& & AWQ+ & \xmark
    & 7.36 & 38.80 & 8.19
    & 13.21 & 59.20 & 14.43 \\
& & SmoothQuant & \xmark
    & 7.10 & 30.67 & 8.12
    & 12.74 & 50.73 & 14.21 \\
& & SINQ+ & \cmark
    & 7.28 & 30.21 & 8.02
    & 12.94 & 49.05 & 14.15 \\
& & QuaRot & \cmark
    & \bfseries 6.97 & 48.47 & 8.86
    & \bfseries 12.53 & 75.07 & 15.11 \\
& & \hl\method{} & \hl\cmark
    & \hl 7.38 & \hl \bfseries 29.91 & \hl \bfseries 7.99
    & \hl 13.11 & \hl \bfseries 47.87 & \hl \bfseries 14.14 \\
\cmidrule(lr){2-10}

& \multirow{7}{*}{\shortstack[l]{A: MXINT4 \\ W: MXINT4 \\ $(G{=}32)$}}
  & RTN & \cmark
    & 7.62 & 48.48 & 8.61
    & 13.60 & 74.73 & 15.22 \\
& & AWQ+ & \xmark
    & 7.52 & 38.61 & 8.44
    & 13.39 & 60.99 & 14.73 \\
& & SmoothQuant & \xmark
    & 7.36 & 33.63 & 8.42
    & 13.08 & 52.34 & 14.75 \\
& & SINQ+ & \cmark
    & 9.65 & 353.83 & 10.55
    & 16.75 & 902.98 & 17.48 \\
& & QuaRot & \cmark
    & 7.64 & 131.31 & 9.95
    & 13.58 & 202.08 & 17.01 \\
& & \hl\method{} & \hl\cmark
    & \hl \bfseries 7.59 & \hl \bfseries 32.99 & \hl \bfseries 8.42
    & \hl \bfseries 13.45 & \hl \bfseries 52.79 & \hl \bfseries 14.79 \\
\cmidrule(lr){2-10}

& \multirow{7}{*}{\shortstack[l]{A: MXFP4 \\ W: MXFP4 \\ $(G{=}32)$}}
  & RTN & \cmark
    & 7.79 & 29.37 & 8.70
    & 13.82 & 46.71 & \bfseries 15.14 \\
& & AWQ+ & \xmark
    & 7.62 & 29.88 & 8.62
    & 13.59 & 50.44 & 15.05 \\
& & SmoothQuant & \xmark
    & 7.68 & 28.66 & 8.53
    & 13.63 & 47.83 & 14.95 \\
& & SINQ+ & \cmark
    & 10.05 & 195.98 & 11.16
    & 17.63 & 395.39 & 18.30 \\
& & QuaRot & \cmark
    & \bfseries 7.38 & 68.96 & 9.45
    & \bfseries 13.22 & 109.84 & 16.40 \\
& & \hl\method{} & \hl\cmark
    & \hl 7.45 & \hl \bfseries 27.46 & \hl \bfseries 8.59
    & \hl 13.49 & \hl \bfseries 42.59 & \hl 15.16 \\
\cmidrule(lr){2-10}

& \multirow{7}{*}{\shortstack[l]{A: NVFP4 \\ W: NVFP4 \\ $(G{=}16)$}}
  & RTN & \cmark
    & 7.07 & 26.12 & 7.93
    & 12.70 & 40.75 & 13.98 \\
& & AWQ+ & \xmark
    & 6.80 & 25.60 & 7.95
    & 12.24 & 41.95 & 13.91 \\
& & SmoothQuant & \xmark
    & 6.70 & 24.16 & 7.73
    & 12.17 & 37.00 & 13.64 \\
& & SINQ+ & \cmark
    & Div$^\dagger$ & Div$^\dagger$ & Div$^\dagger$
    & Div$^\dagger$ & Div$^\dagger$ & Div$^\dagger$ \\
& & QuaRot & \cmark
    & 12.11 & 173.38 & Div$^\dagger$
    & 26.37 & 280.53 & Div$^\dagger$ \\
& & \hl\method{} & \hl\cmark
    & \hl \bfseries 6.53 & \hl \bfseries 23.77 & \hl \bfseries 7.75
    & \hl \bfseries 11.81 & \hl \bfseries 36.64 & \hl \bfseries 13.66 \\

\bottomrule
\end{tabular}%
}
\end{table}

%% file: tables/downstream_qwen_appndx.tex
\begin{table}[t]
\centering
\caption{\green{Qwen3-14B} downstream accuracy (0-shot, \%, higher is better $\uparrow$) for RTN, AWQ+, SmoothQuant, SINQ+, \method{}, and QuaRot, across all 4 weight formats; QuaRot is omitted entirely for NVFP4, since QuaRot's Hadamard rotation redistributes activation magnitude in a way that collapses NVFP4's narrow-range E4M3 block scale. \textbf{Bold}: best calibration-free (CF) result per column within each model+format block. \cmark: calibration-free. \xmark: requires calibration data.}
\label{tab:downstream_qwen}
\setlength{\tabcolsep}{4pt}
\renewcommand{\arraystretch}{1.1}
\resizebox{0.8\textwidth}{!}{%
\begin{tabular}{l l c r r r r r}
\toprule
\multicolumn{8}{c}{\textbf{\green{Qwen3-14B}}} \\
\midrule
Format & Method & CF & MMLU & HellaSwag & PIQA & ARC-Easy & \green{WinoGrande} \\
\midrule

\multirow{7}{*}{INT4 (G=64)}
  & RTN & \cmark &  70.33 & 75.23 & 77.64 &  77.36 & 69.06 \\
& AWQ+ & \xmark & 71.92 & 75.03 & 77.86 & 79.08 & 69.53 \\
& SmoothQuant & \xmark & 73.04 & 76.74 & 77.91 & 80.77 & 71.51 \\
& SINQ+ & \cmark & 72.30 &  76.23 & \bfseries 77.80 & 79.59 & 70.17 \\
& QuaRot & \cmark & 71.30 & 73.52 & 76.77 & 77.99 & 67.88 \\
& \hl\method{} & \hl\cmark & \hl \bfseries 72.33 & \hl \bfseries 76.54 & \hl 76.71 & \hl \bfseries 80.51 & \hl \bfseries 71.09 \\
\midrule
\multirow{7}{*}{MXINT4 (G=32)}
  & RTN & \cmark & 70.08 & 74.31 & 76.69 & 74.74 & 66.08 \\
& AWQ+ & \xmark & 72.42 & 74.93 & 77.37 & 77.53 & 66.93 \\
& SmoothQuant & \xmark & 72.50 & 75.84 & 77.97 & 78.62 & 66.30 \\
& SINQ+ & \cmark & 62.38 & 69.26 & 74.59 & 70.50 & 63.54 \\
& QuaRot & \cmark & 68.72 & 69.92 & 75.63 & 75.72 & 63.85 \\
& \hl\method{} & \hl\cmark & \hl \bfseries 71.47 & \hl \bfseries 75.27 & \hl \bfseries 77.04 & \hl \bfseries 76.68 & \hl \bfseries 70.64 \\
\midrule
\multirow{7}{*}{MXFP4 (G=32)}
  & RTN & \cmark & 72.36 & \bfseries 75.59 & 77.15 & 79.76 & 69.14 \\
& AWQ+ & \xmark & 72.99 & 75.89 & 77.15 & 79.46 & 70.72 \\
& SmoothQuant & \xmark & 73.26 & 75.53 & 77.97 & 79.46 & 69.14 \\
& SINQ+ & \cmark & 62.58 & 68.53 & 74.16 & 73.02 & 63.85 \\
& QuaRot & \cmark & 68.34 & 71.01 & 76.71 & 72.85 & 68.90 \\
& \hl\method{} & \hl\cmark & \hl \bfseries 72.91 & \hl  75.55 & \hl \bfseries 77.48 & \hl \bfseries 80.01 & \hl \bfseries 72.06 \\
\midrule
\multirow{5}{*}{NVFP4 (G=16)}
  & RTN & \cmark & 74.55 &  76.97 &  78.29 &  80.09 &  71.27 \\
& AWQ+ & \xmark & 74.84 & 77.32 & 78.78 & 81.73 & 71.03 \\
& SmoothQuant & \xmark & 75.97 & 77.55 & 78.94 & 82.79 & 72.77 \\
& SINQ+ & \cmark & 23.39 & 26.38 & 50.22 & 26.56 & 50.59 \\
& \hl\method{} & \hl\cmark & \hl \bfseries 75.06 & \hl \bfseries 77.09 & \hl \bfseries 78.99 & \hl \bfseries 81.90 & \hl \bfseries 72.64 \\

\bottomrule
\end{tabular}}
\end{table}

%% file: tables/ppl_8bit_weight_activation.tex
\begin{table}[ht]
\centering
\caption{WikiText-2 perplexity ($\downarrow$) for 8-bit W+A
quantization. \method{} applies $\boldsymbol{\beta}$-rescaling
without calibration data. SmoothQuant uses 256 calibration
samples. SINQ+ is our extension of SINQ to \green{W+A}. \textbf{Bold}: best calibration-free (\textbf{CF}) result.
\cmark: calibration-free. \xmark: requires calibration data.}
\label{tab:ppl_8bit_wa}
\setlength{\tabcolsep}{4pt}
\renewcommand{\arraystretch}{1.1}
\resizebox{.8\textwidth}{!}{%
\begin{tabular}{@{}cllc cccc@{}}
\toprule
\textbf{Bits} & \textbf{Format} &
\textbf{Method} & \textbf{CF} &
\textbf{Qwen2.5-7B} & \textbf{Qwen3-0.6B} &
\textbf{Qwen3-14B} & \textbf{Llama-3.1-8B} \\
\midrule
\multicolumn{3}{l}{BF16 (reference)} & --
  & 6.15 & 18.14 & 7.54 & 5.62 \\
\midrule

\multirow{12}{*}{\rotatebox[origin=c]{90}{\textbf{8-bit}}}

& \multirow{4}{*}{\shortstack[l]{
    A: \hspace{1.6pt}INT8 \\ W: INT8 \\ $(G{=}128)$}}
  & RTN          & \cmark
    & \textbf{6.27} & 18.42 & 7.55 & 5.72 \\
& & SmoothQuant  & \xmark
    & 6.25 & 18.28 & 7.50 & 5.70 \\
& & SINQ+         & \cmark
    & 6.67 & 18.33 & \textbf{7.49} & 5.70 \\
& & \hl\method{}    & \hl\cmark
    & \hl 6.64 & \hl\textbf{18.24} & \hl7.77 & \hl\textbf{5.69} \\
\cmidrule(lr){2-8}

& \multirow{4}{*}{\shortstack[l]{
    A: \hspace{1.6pt}MXINT8 \\ W: MXINT8 \\ $(G{=}32)$}}
  & RTN          & \cmark
    & 6.15 & \textbf{18.06} & 7.54 & 5.63 \\
& & SmoothQuant  & \xmark
    & 6.15 & 18.18 & 7.55 & 5.63 \\
& & SINQ+         & \cmark
    & 7.37 & 47.87 & 9.04 & 9.51 \\
& & \hl\method{}    & \hl\cmark
    & \hl\textbf{6.15} & \hl 18.13 & \hl\textbf{7.53} & \hl\textbf{5.63} \\
\cmidrule(lr){2-8}

& \multirow{4}{*}{\shortstack[l]{
    A: \hspace{1.6pt}MXFP8 \\ W: MXFP8 \\ $(G{=}32)$}}
  & RTN          & \cmark
    & 6.22 & 18.54 & \textbf{7.62} & 5.72 \\
& & SmoothQuant  & \xmark
    & 6.22 & 18.46 & 7.63 & 5.71 \\
& & SINQ+         & \cmark
    & 7.49 & 53.41 & 9.21 & 10.48 \\
& & \hl\method{}    & \hl\cmark
    & \hl\textbf{6.21} & \hl\textbf{18.23} & \hl 7.64 & \hl\textbf{5.71} \\
\bottomrule
\end{tabular}%
}
\end{table}

%% file: appendix/weight_only_results.tex
\section{Weight-only quantization results}
\label{app:weight_only}

While the primary contribution of \method{} is calibration-free
joint W+A quantization, \mbox{Table~\ref{tab:ppl_weight_only}}
evaluates its weight-only performance, with activations
kept in BF16, to assess whether the MSE-based dual scaling
also yields competitive weight reconstruction in isolation.
\method{} \green{achieves the best calibration-free results on nearly all format--model combinations (on Qwen3-14B it is marginally behind SINQ on INT4 and RTN on NVFP4)}, closely matching
calibration-based AWQ on INT4 and MX formats and remaining within a small margin of GPTQ without requiring any calibration data or Hessian computation. 
While GPTQ relies on costly second-order weight updates and is inapplicable when training data is proprietary, \method{} operates purely from weight 
statistics, yet closely follows its perplexity in INT4 and MX formats and 
\green{\emph{outperforms} it on NVFP4 for three of the four models}.

On MXINT4 and MXFP4, \method{} substantially outperforms
SINQ, whose format-unaware scales cause severe degradation
(e.g., $112.00$ and $94.84$ on Qwen3-0.6B for MXINT4 and
MXFP4 respectively). On NVFP4, \method{} achieves the
strongest results among all methods on Qwen2.5-7B ($6.33$)
and Qwen3-0.6B ($19.82$). 


These results confirm that \method{} is highly effective in the  weight-only setting, with MSE minimization directly translating into measurable perplexity gains. More significantly, the same  optimization simultaneously enables calibration-free W+A quantization (Table~\ref{tab:ppl_wa_llama}), distinguishing \method{} from prior weight-only methods.


Importantly, the SINQ results in
Table~\ref{tab:ppl_weight_only} use the original SINQ
algorithm~\citep{muller2025sinq} with no activation scaling
involved. For MX and NV formats the only modification is projecting
SINQ's output scales to the target data format after the
algorithm completes, as there is no other way to deploy
them in these formats. Even in this purely weight-only setting,
\method{} \green{outperforms SINQ on nearly all formats
and models}. Notably, on INT4, where SINQ operates exactly as designed
with unconstrained BF16 scales, \method{} achieves better
perplexity on three out of four models. On MX and NV formats
the gap widens dramatically, confirming that \method{} provides better generalizability to block quantization formats. 

\input{tables/ppl_4bit_weight_only}

%% file: tables/ppl_4bit_weight_only.tex
\begin{table}[ht]
\centering
\caption{WikiText-2 perplexity ($\downarrow$) for 4-bit
weight-only quantization. Activations in BF16.
\method{} applies $\boldsymbol{\beta}$-rescaling without
calibration data. AWQ and GPTQ use 256 calibration samples.
($\dagger$): degenerate result (PPL$>1000$); method incompatible with format.
\textbf{Bold}: best calibration-free (\textbf{CF}) result.
\cmark: calibration-free. \xmark: requires calibration data.}
\label{tab:ppl_weight_only}
\setlength{\tabcolsep}{4pt}
\renewcommand{\arraystretch}{1.1}
\resizebox{0.8\textwidth}{!}{%
\begin{tabular}{@{}cllc cccc@{}}
\toprule
\textbf{Bits} & \textbf{Format} &
\textbf{Method} & \textbf{CF} &
\textbf{Qwen2.5-7B} & \textbf{Qwen3-0.6B} &
\textbf{Qwen3-14B} & \textbf{Llama-3.1-8B} \\
\midrule
\multicolumn{3}{l}{BF16 (reference)} & --
  & 6.15 & 18.14 & 7.54 & 5.62 \\
\midrule

\multirow{20}{*}{\rotatebox[origin=c]{90}{\textbf{4-bit}}}

& \multirow{5}{*}{\shortstack[l]{
    W: INT4 \\ $(G{=}64)$}}
  & RTN          & \cmark
    & 6.67 & 26.33 & 7.93 & 6.32 \\
& & AWQ          & \xmark
    & 6.50 & 21.83 & 7.72 & 6.12 \\
& & GPTQ         & \xmark
    & 6.44 & 20.39 & 7.73 & 5.56 \\
& & SINQ         & \cmark
    & 6.57 & 22.21 & \textbf{7.76} & 6.16 \\
& & \hl\method{}    & \hl\cmark
    & \hl\textbf{6.55} & \hl\textbf{21.35} & \hl7.88 & \hl\textbf{6.12} \\
\cmidrule(lr){2-8}

& \multirow{5}{*}{\shortstack[l]{
    W: MXINT4 \\ $(G{=}32)$}}
  & RTN          & \cmark
    & 6.66 & 25.99 & 7.98 & 6.41 \\
& & AWQ          & \xmark
    & 6.49 & 21.90 & 7.85 & 6.16 \\
& & GPTQ         & \xmark
    & 6.49 & 20.09 & 7.79 & 6.25 \\
& & SINQ         & \cmark
    & 7.98 & 112.00 & 9.67 & 13.09 \\
& & \hl\method{}    & \hl\cmark
    & \hl\textbf{6.60} & \hl\textbf{23.02} & \hl\textbf{7.89} & \hl\textbf{6.19} \\
\cmidrule(lr){2-8}

& \multirow{5}{*}{\shortstack[l]{
    W: MXFP4 \\ $(G{=}32)$}}
  & RTN          & \cmark
    & 6.79 & 22.38 & 8.04 & 6.29 \\
& & AWQ          & \xmark
    & 6.58 & 20.98 & 7.94 & 6.31 \\
& & GPTQ         & \xmark
    & 6.47 & 19.79 & 7.79 & 6.11 \\
& & SINQ         & \cmark
    & 8.03 & 94.84 & 9.85 & 16.16 \\
& & \hl\method{}    & \hl\cmark
    & \hl\textbf{6.63} & \hl\textbf{22.35} & \hl\textbf{7.95} & \hl\textbf{6.19} \\
\cmidrule(lr){2-8}

& \multirow{5}{*}{\shortstack[l]{
    W: NVFP4 \\ $(G{=}16)$}}
  & RTN          & \cmark
    & 6.85 & 23.35 & \textbf{7.81} & 8.12 \\
& & AWQ          & \xmark
    & 6.61 & 22.56 & 7.89 & 7.93   \\
& & GPTQ         & \xmark
    & 6.55 & 20.41 & 7.74 & 7.61 \\
& & SINQ         & \cmark
    & {Div$\dagger$}  & {Div$\dagger$}    & {Div$\dagger$}   & {Div$\dagger$}   \\
& & \hl\method{}    & \hl\cmark
    & \hl\textbf{6.33} & \hl\textbf{19.82} & \hl 8.04 & \hl\textbf{6.43} \\
\bottomrule
\end{tabular}%
}
\end{table}

%% file: appendix/packed_w4a4.tex
\section{Real packed W4A4 deployment}
\label{app:packed-w4a4}

\paragraph{Implementation.}

The packed inference path targets INT4 with group size $G{=}64$ and is built on the packed W4A4 GEMM backend from Nunchaku/SVDQuant~\citep{li2025svdquant}. The quantized weights, $\hat{Q}$, and scales produced by Algorithm~\ref{alg:dual_scale} for all $224$ linear layers of Llama-3.1-8B~\citep{dubey2024llama3} ($7$ projections $\times$ $32$ transformer blocks) are packed offline into the target storage layout and executed through this backend. Embedding, lm\_head, and normalization layers remain in BF16. We extend the backend with two components: (i) a CUDA kernel that incorporates the column scale $\beta$ into activation quantization, avoiding a separate kernel launch, and (ii) an allocation-free GEMM entry point that enables CUDA Graph capture for the full model.

\paragraph{Evaluation setup.}

All measurements are performed on a single NVIDIA RTX 4090~\citep{nvidia_rtx4090} with $24$~GB of memory. WikiText-2 perplexity is evaluated at sequence length $\mathrm{seq}=8192$. Latency and memory are measured using an $82$-token prompt for prefill without a KV cache and $100$ generated tokens for decode, corresponding to a $182$-token KV cache. All reported measurements are collected from warm runs and remain stable within $0.2\%$ across repetitions. 

\paragraph{Accuracy.}
The packed W4A4 implementation achieves a WikiText-2 perplexity of $6.74$, compared with $6.73$ under fake quantization (Table~\ref{tab:ppl_wa_llama}), a difference of only $0.01$. The BF16 baseline measured using the same evaluation harness has a perplexity of $5.61$, compared with $5.62$ reported in Table~\ref{tab:ppl_wa_llama}. 

\paragraph{Latency breakdown.}
Table~\ref{tab:packed-latency} reports the GPU time for individual components under BF16 and \green{\method{}} W4A4 with INT4 ($G{=}64$) format. Quantized weights remain packed throughout inference, and group scales are applied within the GEMM, eliminating runtime weight dequantization. Activation quantization, including the $\beta$ pre-scaling operation, costs $0.70$--$3.09$~ms and accounts for $4.3$--$6.1\%$ of total W4A4 GPU time, constituting the only per-step quantization overhead. The largest reduction comes from the quantized linear layers, whose execution time decreases from $17.33$ to $8.82$~ms for prefill with batch size $1$ and from $121.49$ to $32.98$~ms for prefill with batch size $16$. The
remaining computation consists primarily of components retained in BF16, including lm\_head, attention, normalization, RoPE, residual operations, and KV-cache updates, which limits the achievable decode speedup. Note that component-level timings are collected in a separate instrumented run; the associated measurement overhead makes the totals in Table~\ref{tab:packed-latency} slightly larger than the end-to-end times reported in Table~\ref{tab:packed-w4a4}.

\begin{table}[t]
\centering
\small
\caption{GPU time (ms) per component for BF16 and \method{} W4A4, INT4 ($G{=}64$), on \mbox{Llama-3.1-8B} using an RTX 4090.}
\label{tab:packed-latency}
\resizebox{.9\textwidth}{!}{
\begin{tabular}{lcc|cc|cc|cc}
\toprule
\multirow{2}{*}{\textbf{Component}} & \multicolumn{2}{c|}{\textbf{Prefill ($b{=}1$)}}
& \multicolumn{2}{c|}{\textbf{Prefill ($b{=}16$)}}
& \multicolumn{2}{c|}{\textbf{Decode ($b{=}1$)}}
& \multicolumn{2}{c}{\textbf{Decode ($b{=}16$)}} \\
 & \textbf{BF16} & \hl \textbf{W4A4} & \textbf{BF16} & \hl \textbf{W4A4} & \textbf{BF16} & \hl \textbf{W4A4} & \textbf{BF16} & \hl \textbf{W4A4} \\
\midrule
Quantized linears (224) & 17.33 & \hl 8.82 & 121.49 & \hl 32.98 & 14.75 & \hl 8.24 & 16.32 & \hl 9.18 \\
Activation quantization (incl.\ $\beta$) & -- & \hl 0.82 & -- & \hl 3.09 & -- & \hl 0.70 & -- & \hl 0.72 \\
lm\_head (BF16) & 1.19 & \hl 1.19 & 9.26 & \hl 9.26 & 1.07 & \hl 1.07 & 1.10 & \hl 1.10 \\
Attention & 0.19 & \hl 0.20 & 0.93 & \hl 0.94 & 0.42 & \hl 0.42 & 1.73 & \hl 1.74 \\
Norms / RoPE / KV cache / residual & 2.43 & \hl 2.33 & 14.03 & \hl 13.37 & 1.60 & \hl 1.64 & 4.18 & \hl 4.14 \\
\midrule
\textbf{Total GPU time}
& \textbf{21.14} & \hl \textbf{13.36}
& \textbf{145.71} & \hl \textbf{59.64}
& \textbf{17.84} & \hl \textbf{12.07}
& \textbf{23.33} & \hl \textbf{16.88} \\
\bottomrule
\end{tabular}
}
\end{table}

\paragraph{Memory breakdown.}
\label{app:memory_overhead}

Table~\ref{tab:packed-memory} provides an end-to-end memory breakdown of the \method{} packed deployment for Llama-3.1-8B when quantized to INT4 ($G{=}64$). Quantization reduces the storage for the $224$ linear layers from $13.96$~GB to $3.49$~GB, with an additional $0.22$~GB required for per-group weight scales. Together with the unquantized BF16 embedding and lm\_head ($2.10$~GB), the total weight footprint decreases from $16.20$~GB to $5.92$~GB. The remaining activation and \mbox{KV-cache} components stay in BF16 and therefore have the same footprint in both paths. Prefill uses \texttt{use\_cache=False}, so no KV cache is allocated; its activation footprint is dominated by the $[\mathrm{seq}\times128256]$ logits tensor. Consequently, memory usage for prefill with batch size $1$ is slightly lower than for decode with batch size $1$.

\begin{table}[t]
\centering
\small
\caption{End-to-end memory breakdown (GB) for BF16 and \method{} W4A4 on Llama-3.1-8B using an RTX 4090.}
\label{tab:packed-memory}
\resizebox{.9\textwidth}{!}{
\begin{tabular}{lcc|cc|cc|cc}
\toprule
\multirow{2}{*}{\textbf{Component}} & \multicolumn{2}{c|}{\textbf{Prefill ($b{=}1$)}}
& \multicolumn{2}{c|}{\textbf{Prefill ($b{=}16$)}}
& \multicolumn{2}{c|}{\textbf{Decode ($b{=}1$)}}
& \multicolumn{2}{c}{\textbf{Decode ($b{=}16$)}} \\
 & \textbf{BF16} & \hl \textbf{W4A4} & \textbf{BF16} & \hl \textbf{W4A4} & \textbf{BF16} & \hl \textbf{W4A4} & \textbf{BF16} & \hl \textbf{W4A4} \\
\midrule
Quantized linear weights
& 13.96 & \hl 3.49
& 13.96 & \hl 3.49
& 13.96 & \hl 3.49
& 13.96 & \hl 3.49 \\
Per-group weight scales (BF16)
& -- & \hl 0.22
& -- & \hl 0.22
& -- & \hl 0.22
& -- & \hl 0.22 \\
Embedding + lm\_head (BF16)
& 2.10 & \hl 2.10
& 2.10 & \hl 2.10
& 2.10 & \hl 2.10
& 2.10 & \hl 2.10 \\
Runtime/allocator overhead
& 0.14 & \hl 0.11
& 0.14 & \hl 0.11
& 0.14 & \hl 0.11
& 0.14 & \hl 0.11 \\
\midrule
\textbf{Weights subtotal}
& \textbf{16.20} & \hl \textbf{5.92}
& \textbf{16.20} & \hl \textbf{5.92}
& \textbf{16.20} & \hl \textbf{5.92}
& \textbf{16.20} & \hl \textbf{5.92} \\
First-forward scratch
& 0.10 & \hl 0.30
& 0.10 & \hl 0.30
& 0.10 & \hl 0.30
& 0.10 & \hl 0.30 \\
KV cache
& 0 & \hl 0
& 0 & \hl 0
& 0.02 & \hl 0.02
& 0.38 & \hl 0.38 \\
Other activations
& 0.03 & \hl 0.03
& 0.36 & \hl 0.36
& 0.02 & \hl 0.02
& 0.07 & \hl 0.07 \\
\midrule
\textbf{Total}
& \textbf{16.33} & \hl \textbf{6.25}
& \textbf{16.66} & \hl \textbf{6.58}
& \textbf{16.34} & \hl \textbf{6.26}
& \textbf{16.75} & \hl \textbf{6.67} \\
\bottomrule
\end{tabular}}
\end{table}

%% file: tables/licenses_table.tex

\begin{table}[h]
\centering
\caption{Licenses of existing assets used in this work.}
\label{tab:licenses}
\small
\setlength{\tabcolsep}{4pt}
\resizebox{.9\textwidth}{!}{%
\begin{tabular}{@{}p{1.5cm}p{12cm}@{}}
\toprule
\textbf{Category} & \textbf{Assets and Licenses} \\
\midrule
\textit{Models} & 
Llama-2-7B (Llama 2 Community); 
Llama-3.1-8B (Llama 3.1 Community); 
Llama-3.2-1B (Llama 3.2 Community); 
Qwen2-1.5B, Qwen2.5-7B, Qwen3-0.6B, Qwen3-14B (Apache 2.0) \\
\midrule
\textit{Datasets} & 
WikiText-2 (CC BY-SA 4.0); 
C4 (ODC-BY 1.0); 
MMLU (MIT);
HellaSwag (MIT); 
PIQA (AFL-3.0);
ARC-Easy (CC BY-SA 4.0);
WinoGrande (CC-BY) \\
\bottomrule
\end{tabular}
}
\end{table}